\documentclass{article}
\usepackage[utf8]{inputenc}

\usepackage{amssymb}
\usepackage{amsmath}
\usepackage{amsthm}
\usepackage{latexsym}
\usepackage{graphicx}
\usepackage{color}
\usepackage{graphics}
\usepackage{natbib}
\PassOptionsToPackage{round}{natbib}
\usepackage{cancel}
\usepackage{thm-restate}
\usepackage{mathtools}
\usepackage[mathscr]{euscript}
\usepackage{hyperref}
\usepackage{bm}
\usepackage{fullpage}
\usepackage{enumitem}
\usepackage{subcaption}

\usepackage{MnSymbol}
\DeclareMathAlphabet\mathbb{U}{msb}{m}{n}
\usepackage{xpatch}

\def\Rset{\mathbb{R}}

\DeclareMathOperator*{\conv}{conv}

\DeclareMathOperator*{\cone}{cone}
\DeclareMathOperator*{\vecspan}{span}

\DeclareMathOperator{\Reg}{\mathsf{Reg}}
\DeclareMathOperator{\WReg}{\mathsf{WReg}}

\DeclareMathOperator{\Rate}{\mathsf{Rate}}
\DeclareMathOperator{\AppLoss}{\mathsf{AppLoss}}
\DeclareMathOperator{\AppDist}{\mathsf{AppDist}}

\DeclareMathOperator{\Id}{Id}

\newtheorem{theorem}{Theorem}

\newtheorem{lemma}{Lemma}

\newcommand{\Lin}{\mathsf{Lin}}
\newcommand{\Aff}{\mathsf{Aff}}

\newcommand{\cA}{\mathcal{A}}
\newcommand{\cB}{\mathcal{B}}
\newcommand{\cC}{\mathcal{C}}

\newcommand{\cI}{\mathcal{I}}

\newcommand{\cL}{\mathcal{L}}

\newcommand{\cP}{\mathcal{P}}

\newcommand{\cS}{\mathcal{S}}

\newcommand{\cU}{\mathcal{U}}

\newcommand{\cX}{\mathcal{X}}
\newcommand{\cY}{\mathcal{Y}}

\newcommand{\cPhi}{\Phi}

\newcommand{\bu}{{\mathbf u}}

\newcommand{\bp}{{\mathbf p}}
\newcommand{\bell}{{\bm \ell}}

\newcommand{\eps}{\varepsilon}

\newcommand{\ignore}[1]{}

\hypersetup{
  colorlinks   = true, 
  urlcolor     = blue, 
  linkcolor    = blue, 
  citecolor   = blue 
}

\title{Rate-Preserving Reductions for Blackwell Approachability}

\author{Christoph Dann \thanks{Google Research, {\tt cdann@cdann.net}} 
\and Yishay Mansour \thanks{Tel Aviv University and  Google Research, {\tt mansour.yishay@gmail.com}. Supported in part by funding from the European Research Council (ERC) under the European Union’s Horizon 2020 research and innovation program (grant agreement No.\ 882396), by the Israel Science Foundation, the Yandex Initiative for Machine Learning at Tel Aviv University and a grant from the Tel Aviv University Center for AI and Data Science (TAD).} 
\and Mehryar Mohri \thanks{Google Research and Courant Institute of Mathematical Sciences, New York, {\tt mohri@google.com}} 
\and Jon Schneider \thanks{Google Research, {\tt jschnei@google.com}} 
\and Balasubramanian Sivan \thanks{Google Research, {\tt balusivan@google.com}}}

\begin{document}

\maketitle

\begin{abstract}
Abernethy et al.\ (2011) showed that Blackwell approachability and no-regret learning are equivalent, in the sense that any algorithm that solves a specific Blackwell approachability instance can be converted to a sublinear regret algorithm for a specific no-regret learning instance, and vice versa. In this paper, we study a more fine-grained form of such reductions, and ask when this translation between problems preserves not only a sublinear rate of convergence, but also preserves the optimal rate of convergence. That is, in which cases does it suffice to find the optimal regret bound for a no-regret learning instance in order to find the optimal rate of convergence for a corresponding approachability instance? 

We show that the reduction of Abernethy et al.\ (2011) does not preserve rates: their reduction may reduce a $d$-dimensional approachability instance $\mathcal{I}_1$ with optimal convergence rate $R_1$ to a no-regret learning instance $\mathcal{I}_2$ with optimal regret-per-round of $R_2$, with $R_{2}/R_{1}$ arbitrarily large (in particular, it is possible that $R_1 = 0$ and $R_{2} > 0$). On the other hand, we show that it is possible to tightly reduce any approachability instance to an instance of a generalized form of regret minimization we call \emph{improper $\phi$-regret minimization} (a variant of the $\phi$-regret minimization of Gordon et al.\ (2008) where the transformation functions may map actions outside of the action set). 

Finally, we characterize when linear transformations suffice to reduce improper $\phi$-regret minimization problems to standard classes of regret minimization problems (such as external regret minimization and proper $\phi$-regret minimization) in a rate preserving manner. We prove that some improper $\phi$-regret minimization instances cannot be reduced to either subclass of instance in this way, suggesting that approachability can capture some problems that cannot be easily phrased in the standard language of online learning.
\end{abstract}

\section{Introduction}

Blackwell's Approachability Theorem is a fundamental result in game theory with far-reaching applications in machine learning, economics, and optimization. It provides a framework for analyzing repeated games where the payoff is a vector rather than a single scalar value. In essence, the theorem allows players to determine whether a specific set of payoff vectors can be ``approached'' on average over time, even if achieving them individually in a single round is impossible (with more sophisticated variants of this theorem characterizing the rate at which this set can be approached). This concept has been instrumental in developing robust algorithms for online learning, strategic decision-making in economic models, and solving complex optimization problems where the objective is multi-dimensional.

A perhaps even more prevalent paradigm in the area of online learning is that of regret minimization. In the most fundamental form of the problem (external regret minimization), a learner wants to take a sequence of actions that performs at least as well as the optimal static action they could have taken in hindsight. One of the central results of the field of online learning is that there exist learning algorithms which achieve regret that is sublinear in the time horizon $T$. 
Understanding how to optimize these regret bounds is a topic of active research, and optimal regret bounds have been established for a variety of problems in this area.

\citet{AbernethyBartlettHazan2011} showed that these two problems -- Blackwell approachability and regret minimization -- are very closely linked, and in fact are ``equivalent’’ in the following sense: given an instance of Blackwell approachability (a vector-valued payoff function and a set to approach), it is possible to reduce it to an instance of regret minimization (specifically, an instance of online linear optimization) such that any sublinear regret algorithm for the regret minimization instance can be used to solve the corresponding Blackwell approachability instance, giving an algorithmic proof of convergence. Conversely, any regret minimization instance can itself be viewed as an instance of Blackwell approachability, so any algorithmic approach to generic Blackwell approachability instances can be applied to solve regret minimization.

In this paper, we explore whether such an equivalence holds at a more fine-grained level, when we care about the specific convergence rates for Blackwell approachability and the specific regret bounds for regret minimization. In other words, imagine we have a problem that we can cast as a specific instance of Blackwell approachability, and we want to understand the optimal convergence rates possible for this instance (this is a common desideratum for many of the aforementioned applications of approachability). Is it sufficient to just understand the optimal regret bounds for the corresponding instance of regret minimization under the reduction of \cite{AbernethyBartlettHazan2011}?

\subsection{Our results}

We begin by answering this question in the negative: there is no direct correspondence between the optimal rate achievable in an approachability instance $\cI$ and the optimal regret bound achievable in the corresponding regret minimization instance $\cI'$. More precisely, the algorithmic construction of \cite{AbernethyBartlettHazan2011} shows that any upper bound on the regret of $\cI’$ translates to an upper bound on the optimal convergence rate of $\cI$. We prove that this translation is lossy, and that there is no analogous translation between lower bounds. In particular, in Theorem \ref{thm:classic-reduction-not-tight} we exhibit instances of approachability $\cI$ that are \emph{perfectly approachable} (i.e., where the learner can guarantee that the average payoff exactly lies within the set $S$) but where the corresponding regret-minimization instance $\cI'$ has a non-trivial regret lower bound.

On the other hand, we complement this result by showing that Blackwell approachability can be tightly reduced (in a rate-preserving manner) to a novel variant of regret minimization that we call \emph{improper $\phi$-regret minimization} (Theorem \ref{thm:app-to-improper}). To explain what this means, it is helpful to briefly define regret minimization and its relevant variants (we defer formal definitions to Section \ref{sec:model}).

We start with the basic setting of online linear optimization. In this problem, a learner must select an action $p_t$ from a  convex action set $\cP \subseteq \Rset^d$ each round $t$ for a total of $T$ rounds. At the same time, an adversary selects a loss $\ell_t$ from a convex loss set $\cL \subseteq \Rset^d$. The learner receives  loss $\langle p_t, \ell_t\rangle$ during that round, and would like to minimize their total \emph{regret} over all rounds. In its most standard form (external regret), this is just the largest gap between their total loss and the total loss of the best fixed action, and can be written as
$$\Reg(\bp, \bell) = \max_{p^{*} \in \cP} \left(\sum_{t=1}^{T}\langle p_{t}, \ell_{t} \rangle - \sum_{t=1}^{T} \langle p^{*}, \ell_{t}\rangle\right).$$

Some applications call for obtaining low regret not just with respect to the best fixed action in hindsight, but additionally with respect to transformations of the sequence of actions played by the learner. The most well-known such notion of regret is probably that of \emph{swap regret}, which famously has the property that sublinear swap regret learning algorithms converge to correlated equilibria when used to play normal-form games \citep{FosterVohra1997}. However, swap regret is just one of a large class of such regret metrics that are succinctly captured by the notion of (linear) \emph{$\phi$-regret} introduced in \citep{GordonGreenwaldMarks2008}. In $\phi$-regret minimization, we are given a collection $\cPhi$ of linear ``transformation'' functions $\phi$ sending $\cP$ to $\cP$. The $\phi$-regret for this class $\Phi$ is given by:
$$\Reg_{\Phi}(\bp, \bell) = \max_{\phi \in \cPhi} \left(\sum_{t=1}^{T}\langle p_{t}, \ell_{t} \rangle - \sum_{t=1}^{T} \langle \phi(p_t), \ell_{t}\rangle\right).$$

Note that by setting $\cPhi$ to be the set of all constant functions on $\cP$, we immediately recover the original external regret definition. The case of swap regret corresponds to the setting where $\cP$ is the $d$-simplex, and $\cPhi$ contains all row-stochastic linear maps. Interestingly, in \citep{GordonGreenwaldMarks2008}, the authors show that any $\phi$-regret minimization problem has a corresponding sublinear regret learning algorithm (this can also be seen from Blackwell approachability). 

However, the constraint that every linear function $\phi$ in $\cPhi$ maps $\cP$ to itself turns out not to be strictly necessary -- there exist classes $\cPhi$ of linear functions that map points in $\cP$ to arbitrary points in $\Rset^d$ (including possibly outside of $\cP$) for which it is still possible to obtain sublinear $\phi$-regret. We call any such instance an \emph{improper $\phi$-regret minimization} instance (and contrast it with the previous constrained definition by calling those instances \emph{proper $\phi$-regret minimization}). In Theorem~\ref{thm:app-to-improper}, we show that for any Blackwell approachability instance $\cI$, there exists an improper $\phi$-regret minimization instance $\cI'$, such that if you have an algorithm that solves $\cI'$ with at most $R$ $\phi$-regret, you can use it to get a convergence rate of $R/T$ in the approachability instance $\cI$, and vice versa.

The introduction of improper $\phi$-regret raises a natural question: is improper $\phi$-regret minimization truly a more general setting than proper $\phi$-regret minimization? Or can we reduce -- in a rate-preserving manner -- any improper $\phi$-regret minimization problem to a proper $\phi$-regret minimization problem (or even further, to an external regret minimization problem, as the original reduction of \cite{AbernethyBartlettHazan2011} partially accomplishes).

Since arbitrary reductions between instances of regret-minimization can be ill-behaved, we study the above questions in the setting of reductions that are entirely specified by linear transformations, which we call \emph{linear equivalences}. It turns out that there are interesting classes of regret-minimization problems that superficially appear to be improper $\phi$-regret minimization instances, but that can be shown to be linearly equivalent to proper $\phi$-regret minimization or external regret minimization problems. One interesting such class is the class of \emph{weighted regret minimization} problems, where the regret corresponding to the transformation function $\phi$ is weighted by a positive scalar $w_{\phi}$ (this class is discussed in Section \ref{sec:weighted} as a motivating example). 

Nonetheless, we show that these three classes of regret minimization problems -- external regret, proper $\phi$-regret, and improper $\phi$-regret -- are all distinct (with each class strictly contained in the next). Specifically, we prove the following results:

\begin{itemize}
    \item We give a clean mathematical characterization of when an improper $\phi$-regret instance is linearly equivalent to some external regret instance (Theorem \ref{thm:external-char}). This characterization puts rather strong constraints on the set of functions $\cPhi$ (in particular, any two functions in $\cPhi$ must differ by a constant), and rules out the possibility of for example swap regret minimization being linearly equivalent to an external regret minimization instance.
    \item We provide examples of improper $\phi$-regret instances that are provably \emph{not} linearly equivalent to any proper $\phi$-regret instance (Section \ref{sec:counterexamples}). This provides evidence that the language of approachability is strictly more powerful than the language of standard regret minimization. These counterexamples also seem mathematically quite rich, connected to concepts like finding linear subspaces of non-invertible matrices.
    
    \item Finally, we provide an algorithmic characterization of when an improper $\phi$-regret instance is linearly equivalent to a proper $\phi$-regret instance, reducing the problem to checking whenever a certain convex cone of $d$-by-$d$ matrices contains any invertible matrices (Section \ref{sec:algorithmic-characterization}). In the case where the action set $\cP$ and $\cPhi$ are specified as the convex hull of $N$ pure actions and $M$ transformation functions respectively, this check can be performed by a randomized algorithm in time polynomial in $N$, $M$, and $d$ (and also provides an effective algorithm for producing such a reduction). 
\end{itemize}

\begin{figure}
    \centering
    \includegraphics[trim={0 17cm 0 0},clip,width=0.95\textwidth]{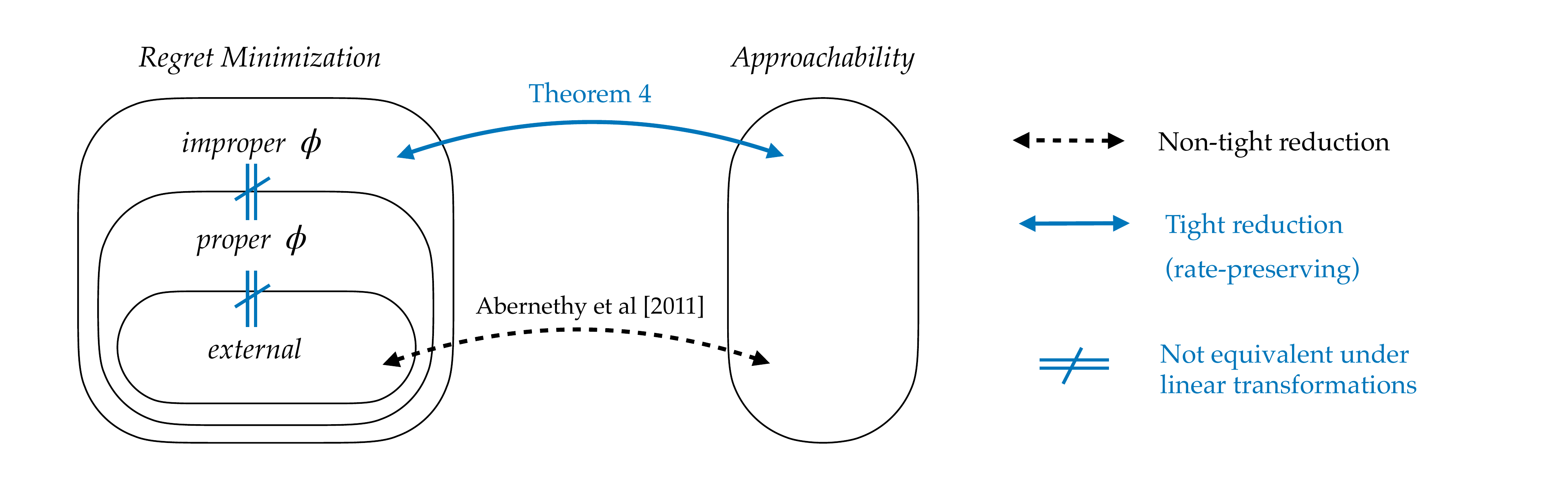}
    \caption{Overview of problem classes and reductions.~\citet{AbernethyBartlettHazan2011} give a non-tight reduction between approachability and external regret minimization. We give in Theorem~\ref{thm:app-to-improper} a tight (rate-preserving) reduction between approachability and the class of improper $\phi$-regret minimization problems. Further, we characterize when an improper $\phi$-regret instance is tightly reducible, under linear equivalence, to more well-studied classes of regret minimization problems, like proper $\phi$-regret minimization (Theorem~\ref{thm:wreg-reduction}) and external regret minimization (Theorem~\ref{thm:external-char}).
    }\label{fig:reduction_overview}
\end{figure}

\begin{figure}
\centering
\begin{tabular}{c|c|c}
  {$\begin{aligned}
    \phi_1(p) &= (1, -p_2, -2p_3) \\
    \phi_2(p) &= (0, 2-p_2, -2p_3) \\
    \phi_3(p) &= (0, -p_2, 3-2p_3)
  \end{aligned}$} &

  {$\begin{aligned}
    \phi_1(p) &= (1, 0, 0) \\
    \phi_2(p) &= (-p_1, 2-p_2, -p_3) \\
    \phi_3(p) &= (-2p_1, -2p_2, 3-2p_3)
  \end{aligned}$} &

  {$\begin{aligned}
    \phi_1(p) &= (p_1-p_2, p_1+p_2, p_3) \\
    \phi_2(p) &= (p_1+p_3, p_2, p_3 - p_1) \\
    \phi_3(p) &= (p_1, p_2 - p_3, p_2 + p_3).
  \end{aligned}$} \\
  (a) & (b) & (c) \\
\end{tabular}
\caption{Examples of improper $\phi$-regret minimization}\label{fig:examples}
\end{figure}

\subsection{Illustrative examples of improper $\phi$-regret and reductions}

To illustrate the idea of linear equivalences between different regret minimization problems, in this section we present three different examples of $\phi$-regret minimization problems. All examples will have action set $\cP = \Delta_3$ (the simplex over $3$ actions) and loss set $\cL = [-1, 1]^d$. The set of transformations $\cPhi$ will always be of the form $\conv(\{\phi_1, \phi_2, \phi_3\})$ but with these specific functions changing between examples as shown in Figure~\ref{fig:examples}.

To begin, note that all these instances are truly \emph{improper} $\phi$-regret problems; it is easy to find a $p = (p_1, p_2, p_3) \in \Delta_3$ for which some $\phi_i(p)$ lies outside $\Delta_3$. Nonetheless, Blackwell approchability can be used to show that you can achieve sublinear $\phi$-regret for all three of these instances.

These three instances have different reducibility properties: instance (a) is linearly equivalent to an external regret problem, instance (b) is linearly equivalent to a proper $\phi$-regret problem (but not to any external regret problem), and instance (c) cannot be reduced to a proper $\phi$-regret problem. 

To give an explicit example of what such a linear reduction entails, for instance (a), it turns out that for any $\ell \in \cL$ and $p \in \cP$ we have the equality $\langle p - \phi_i(p), \ell \rangle = \langle p - e_i, \ell'\rangle$, where $e_i$ is the $i$th basis vector and $\ell' = \mathrm{diag}(1, 1/2, 1/3)\ell$. Therefore, this strange improper $\phi$-regret minimization problem is really an ordinary external regret minimization problem (on the modified loss set $\cL' = \mathrm{diag}(1, 1/2, 1/3)\cL$) in disguise. One can also check that the different $\phi_i(p)$ differ by a constant, as required by our characterization.

Similarly, for instance (b), one can check that although the $\phi_i$ are improper, we can write $\langle p - \phi_i(p), \ell\rangle = \langle p - \phi'_i(p), \ell'\rangle$ for a set of proper $\phi'_i$ and where $\ell' = \ell/3$. On the other hand, there is no way to reduce this instance to an external regret instance, since different $\phi_i$ do not differ by a constant function. Interestingly, for this instance we can also write $\langle p - \phi_i(p), \ell\rangle = w_i \cdot \langle p - e_i, \ell \rangle$ (for $(w_1, w_2, w_3) = (1, 2, 3)$), showing that this instance captures a version of \emph{weighted external regret} (where regret w.r.t action $i$ is multiplied by $w_i$). In particular, this demonstrates the (somewhat surprising fact) that weighted external regret cannot be written exactly as an external regret minimization problem.

Finally, instance (c) is an improper $\phi$-regret minimization problem that provably cannot be reduced to either of the smaller classes. Interestingly the construction of this $\phi$ comes from the fact that the skew-symmetric matrices of odd dimension form a linear subspace of singular matrices (each $p - \phi_i(p)$ is a skew-symmetric linear transformation).

\subsection{Prior work}

\paragraph{Blackwell's approachability, applications and extensions.} There is a large body of work studying, extending, and applying \citeauthor{Blackwell1956}'s approachability theory to various problems of interest, including regret minimization~\citep{FosterVohra1999}, game theory~\citep{HartMasColell2000}, reinforcement learning~\citep{MannorShimkin2003,KalathiBorkarJain2014,MiryoosefiBrantleyDaumeDudikSchapire2019}, calibration~\citep{Dawid82,FosterHart2017}. Approachability and partial monitoring were studied in a series of publications by \citet{Perchet2010,MannorPerchetStoltz2014,
  MannorPerchetStoltz2014Bis,
  PerchetQuincampoix2015,PerchetQuincampoix2018,KwonPerchet2017}. More recently, approachability has also been used in the analysis of fairness in machine learning \citep{ChzhenGiraudStoltz2021}. The notion of approachability has been extended in several studies. These include
\citet{Vieille1992} studying weak approachability in finite dimensional
spaces,~\citet{Spinat2002} providing necessary and sufficient approachability conditions for arbitrary sets (not just convex sets), \citet{Lehrer2003Bis} extending approachability theory to infinite-dimensional spaces.

\paragraph{Improved regret guarantees through approachability for other norms.} As mentioned earlier,~\citet{AbernethyBartlettHazan2011} showed how to use Blackwell's approachability to solve a general class of regret minimization problems. Nevertheless, their reduction could lead to suboptimal regret guarantees: e.g., a $\sqrt{TK}$ regret for an external regret minimization problem (with $T$ steps and $K$ actions), where $\sqrt{T\log(K)}$ is the optimal regret. It has been observed by many papers~\citep{Perchet2015,Shimkin2016,Kwon2021,DannMansourMohriShneiderSivan2023} since then that this is due to the choice of the $\ell_2$ norm for measuring the distance between the average payoff and the target set in approachability. With approachability algorithms for the more suitable $\ell_\infty$ norm, one can recover the optimal regret guarantees for many problems.~\cite{DannMansourMohriShneiderSivan2023} address the time and space complexity of such algorithms from being prohibtibitely high: i.e., from not depending polynomially on the dimension of the space of vector payoffs.

\paragraph{Swap regret and repeated games.} The study of swap-regret and its generalizations has seen renewed interest in recent years due to its interesting connections to learning agents playing repeated games against strategic agents. While there is a large body of work on strategic agents playing against each other in repeated games, and also learning agents playing against each other in repeated games, the outcome of strategic agents playing against learners has remained largely unexplored until recently. The work of \citet*{DSS19} initiate the study of optimizer-learner interactions and show that a learner playing a no-swap-regret learning algorithm in a repeated game will not let an optimizer's reward exceed the Stackelberg value of the game, where the latter itself is always obtainable by an optimizer playing against any no-regret learner. \citet*{MMSS22} show that a learner playing a no-swap-regret algorithm is not just sufficient, but also necessary to ensure that an optimizer gets no more than the Stackelberg value of the game. They further study the class of Bayesian games and give a sufficient condition in the form of no-polytope-swap-regret for the optimizer to not exceed the Bayesian Stackelberg value of the game. This condition of no-polytope-swap-regret was shown to be necessary to cap the optimizer utility at the Bayesian Stackelberg value by \citet{RZ24}. Special classes of these include \citep{BMSW18} studying the specific 2-player Bayesian game of an auction between a single seller and single buyer, \citet*{ADMS18} studying a similar setting but also other buyer behaviors beyond learning, and \citet*{CWWZ23} extending these to a single seller and multiple buyers.

\paragraph{$\Phi$-regret.} Several prominent notions of regret are special instances of $\Phi$-regret \citep{GreenwaldJafari2003}. These include standard external regret, internal regret \citep{FosterVohra1997} (see also \citep{StoltzLugosi2005,StoltzLugosi2007,GreenwaldLiSchudy2008}), and swap regret \citep{BlumMansour2007, peng2023fast, dagan2023external}. The notions of conditional swap regret \citep{MohriYang2014} or transductive regret \citep{MohriYang2017} are also related to $\Phi$-regret. Recently, $\Phi$-regret minimization has found applications in designing learning algorithms for extensive-form games that converge to certain classes of correlated equilibria \citep{celli2021decentralized, bai2022efficient, zhang2024efficient}.

\section{Model and Preliminaries}
\label{sec:model}

Given two convex sets $\cC_1 \subseteq \Rset^{d_1}$ and $\cC_2 \subseteq \Rset^{d_2}$, we define their tensor product $\cC_1 \otimes \cC_2$ to be the subset of $\Rset^{d_1} \otimes \Rset^{d_2} \simeq \Rset^{d_1d_2}$ equal to the convex hull of all vectors of the form $c_1 \otimes c_2 = c_1 c_2^\top$ for $c_1 \in \cC_1$ and $c_2 \in \cC_2$. We write $\Lin(\cC_1, \cC_2)$ to denote the collection of linear maps that send every point in $\cC_1$ to a point in $\cC_2$, and $\Aff(\cC_1, \cC_2)$ to denote the collection of affine maps that send every point in $\cC_1$ to a point in $\cC_2$. Note that if $\cC_1$ belongs to an affine subspace of $\Rset^{d_1}$ (that does not contain the origin), then any affine function on $\cC_1$ can be equivalently written as a linear function and so $\Aff(\cC_1, \cC_2) \simeq \Lin(\cC_1, \cC_2)$. On the other hand, if this is not the case, we can always augment $\cC_1$ (by adding a $(d+1)$st dummy coordinate) and have that $\Aff(\cC_1, \cC_2) \simeq \Lin(\cC'_1, \cC_2)$. 

\subsection{Approachability}

An \textit{approachability problem}\footnote{We will assume here that we are only concerned with approachability problems where the goal is to approach the negative orthant in $\ell_{\infty}$-distance. It is straightforward to show that any approachability problem can be reduced to an equivalent approachability problem of this form (for completeness, we show this in Appendix \ref{app:orthant}).} is defined by three bounded convex subsets of Euclidean space: an action set $\cP$, a loss set $\cL$, and a ``constraint'' set of bi-affine functions $\cU \subset \Aff(\cP \otimes \cL, \Rset)$. Traditionally, $\cU$ is provided in the form of a multidimensional bilinear function $\bu: \cP \times \cL \rightarrow \Rset^{n}$ (which would correspond to the $\cU$ given by the convex hull of the $n$ components of $\bu$), but it is easier and more general to work with $\cU$ directly. For example, this lets us easily consider approachability problems with infinite-dimensional $\bu$.  

For a given time horizon $T$, an \textit{approachability algorithm} $\cA$ is a collection of functions $\{\cA_{t}\}_{1 \leq t \leq T}$, each of which takes as input a prefix of losses $\ell_1, \ell_2, \dots, \ell_{t-1} \in \cL$ and returns an action $\cA_t(\ell_1, \dots, \ell_{t-1}) = p_t$. For a given sequence of $T$ actions $\bp$ and $T$ losses $\bell$, we define the \textit{approachability loss} $\AppLoss(\bp, \bell)$ to be

\begin{equation}\label{eq:app_loss_def}
\AppLoss(\bp, \bell) = \max_{u \in \cU}\sum_{t=1}^{T} u(p_t, \ell_t). \end{equation}

The objective of an approachability problem is to construct approachability algorithms which guarantee low $\AppLoss$. For a given algorithm $\cA$ and time horizon $T$, let $\AppLoss_{T}(\cA)$ be the worst-case approachability loss of $\cA$ on any sequence of $T$ losses, that is, $\AppLoss_{T}(\cA) = \max_{\bell \in \cL^{T}} \AppLoss(\cA(\bell), \bell)$. We will omit the subscript $T$ when it is clear from context.

We say an approachability instance $(\cP, \cL, \cU)$ is \textit{approachable} if for each $\ell \in \cL$, there exists a $p \in \cP$ such that $u(p, \ell) \leq 0$ for all $u \in \cU$. Blackwell's theorem \citep{Blackwell1956} provides the following dichotomy:

\begin{enumerate}
    \item If an approachability instance $(\cP, \cL, \cU)$ is \textbf{not} approachable, then for any algorithm $\cA$, $\AppLoss_{T}(\cA) = \Omega(T)$.
    \item If an approachability instance $(\cP, \cL, \cU)$ is approachable, then there exists an algorithm $\cA$ where $\AppLoss_{T}(\cA) = O(\sqrt{T})$.
\end{enumerate}

For an approachable instance $(\cP, \cL, \cU)$, define

\begin{equation}\label{eq:def_alpha}
    \Rate(\cP, \cL, \cU) = \lim_{T\rightarrow\infty} \inf_{\cA} \frac{\AppLoss_{T}(\cA)}{\sqrt{T}}. 
\end{equation}

In words, the optimal algorithm for the instance $(\cP, \cL, \cU)$ achieves a worst-case approachability loss of $\Rate(\cP, \cL, \cU)\sqrt{T} + o(\sqrt{T})$.

\subsection{Regret minimization}\label{sec:regret-minimization}

Much like approachability problems, we specify a \emph{regret minimization problem} by a triple of bounded convex sets: an action set $\cP \subseteq \Rset^d$, a loss set $\cL \subseteq \Rset^{d}$, and a benchmark set $\cPhi \subseteq \Aff(\cP, \Rset^d)$ of affine functions from $\cP$ to $\Rset^{d}$). Intuitively, each function $\phi \in \cPhi$ represents a ``swap'' or ``transformation'' function that the learner is competing against: if the learner outputs a sequence of actions $p_1, p_2, \dots, p_T$, they want to have low regret compared to the sequence of actions $\phi(p_1), \phi(p_2), \dots, \phi(p_T)$. 

For a given regret minimization instance $(\cP, \cL, \cPhi)$, the regret of a sequence of $T$ actions $\bp$ and $T$ losses $\bell$ is given by
\begin{equation}\label{eq:phi-regret}
\Reg(\bp, \bell) = \max_{\phi \in \Phi}\left(\sum_{t=1}^{T} \langle p_t, \ell_t\rangle - \sum_{t=1}^{T} \langle \phi(p_t), \ell_t \rangle\right).
\end{equation}
As with approachability, our goal is to design a learning algorithm $\cA$ for this problem which achieves low worst-case regret. In fact, it is fairly straightforward to write the above regret minimization problem as a specific instance of approachability: let $\cU$ be the set of bilinear functions of the form

\begin{equation} \label{eq:reg-to-app}
    u_{\phi}(p, \ell) = \langle p, \ell\rangle - \langle \phi(p), \ell\rangle,
\end{equation}

\noindent
where $\phi$ ranges over all $\phi \in \cPhi$, then it is easy to see that the expression for $\Reg(\bp, \bell)$ in \eqref{eq:phi-regret} is equivalent to the expression for $\AppLoss(\bp, \bell)$ in \eqref{eq:app_loss_def}. We will likewise write $\Reg_{T}(\cA)$ to refer to the worst-case regret of algorithm $\cA$, and $\Rate(\cP, \cL, \cPhi)$ to represent the optimal regret rate for this regret minimization instance. 

If each $\phi \in \cPhi$ contains a fixed point in $\cP$, then the corresponding approachability problem is approachable and $\Rate(\cP, \cL, \cPhi) < \infty$. Furthermore, if the loss set $\cL$ contains a vector in every direction ($\cone(\cL) = \Rset^d$), then this condition exactly characterizes regret minimization instances that permit sublinear regret algorithms.

\begin{theorem}\label{thm:reg-to-app}
Consider a regret minimization instance $(\cP, \cL, \cPhi)$. If each $\phi \in \cPhi$, has a fixed point $p_{\phi} \in \cP$ (i.e., $\phi(p_{\phi}) = p_{\phi}$), then $\Rate(\cP, \cL, \cPhi) < \infty$. Conversely, if $\cone(\cL) = \Rset^d$ and $\Rate(\cP, \cL, \cPhi) < \infty$, then every $\phi \in \cPhi$ must have a fixed point in $\cP$.
\end{theorem} 
\begin{proof}
To prove the forward direction, we will prove the corresponding approachability problem (defined via \eqref{eq:reg-to-app}) is approachable. For this, we must show that for any $\phi \in \cPhi$, there exists a $p \in \cP$ such that $u_{\phi}(p, \ell) \leq 0$ for all $\ell \in \cL$. But since $u_{\phi}(p, \ell) = \langle p - \phi(p), \ell\rangle$, if we take $p = p_{\phi}$, $p_{\phi} - \phi(p_{\phi}) = 0$ and $u_{\phi}(p, \ell) = 0$ for all $\ell \in \cL$.

Conversely, assume there exists a $\phi \in \cPhi$ without a fixed point. Consider the set $V = \{p - \phi(p) \mid p \in \cP\}$. $V$ is a convex bounded set (since it is a linear transformation of $\cP$ that by assumption does not contain $0$. Therefore, there exists a hyperplane separating $V$ from $0$ -- in particular, there exists a $w$ such that $\langle v, w \rangle > 0$ for any $v \in V$. If we take $\ell$ proportional to $w$ (possible since $\cone(\cL) = \Rset^d$), then for this $\phi$ and $\ell$, there is no $p \in \cP$ such that $u_{\phi}(p, \ell) \leq 0$. It follows that the corresponding approachability problem is not approachable, and therefore $\Rate(\cP, \cL, \cPhi) = \infty$.
\end{proof}

We will categorize regret minimization instances into three classes based on the properties of their set of benchmarks $\cPhi$, which we list in increasing order of generality:

\begin{itemize}
    \item \textbf{External regret minimization}. If each $\phi \in \cPhi$ has the property that $\phi$ is a constant function over $\cP$ (i.e., there exists a $p_{\phi} \in \cP$ s.t. $\phi(p) = p_{\phi}$ for all $p \in \cP$), then we say that this instance of regret minimization is an \emph{external regret minimization} instance. Note that the classic online learning setting of \emph{online linear optimization} fits into this class (the benchmark set $\cPhi$ is the set of all constant functions on $\cP$).
    \item \textbf{Proper $\phi$-regret minimization}. If each $\phi \in \cPhi$ has the property that $\phi(p) \in \cP$ whenever $p \in \cPhi$ (i.e., each $\phi$ maps $\cP$ into itself), then we say that this instance of regret minimization is a \emph{proper $\phi$-regret minimization} instance. Note that by Brouwer's fixed-point theorem, each $\phi$ must contain a fixed point in $\cP$, and therefore each proper $\phi$-regret minimization problem has a sublinear-regret algorithm. 
    
    This class captures the well-known cases of swap regret and internal regret, along with the notion of linear $\phi$-regret studied in \citep{GordonGreenwaldMarks2008}. It also contains the previous class of external regret minimization instances.
    \item \textbf{Improper $\phi$-regret minimization}. As long as each $\phi \in \cPhi$ has a fixed point in $\cP$ (a $p \in \cP$ s.t. $\phi(p) = p$), we say that this instance of regret minimization is an \emph{improper $\phi$-regret instance}. 
    
    Note here that the benchmark functions $\phi$ are allowed to send points in $\cP$ to points outside of $\cP$ (that the learner is not even allowed to play). Nonetheless, by Theorem \ref{thm:reg-to-app}, we know that $\Rate(\cP, \cL, \cPhi) < \infty$ for any such improper $\phi$-regret instance.
\end{itemize}

\subsection{Reductions between learning problems}

 Ideally, we would design efficient learning algorithms that provably match the optimal rate for a given approachability or regret minimization instance. Unfortunately, doing this directly seems very challenging -- even for the special case of online linear optimization, it is unclear how to construct efficient learning algorithms with near-optimal regret bounds. 

Instead, we will settle for being able to reduce an arbitrary approachability instance $(\cP, \cL, \cU)$ to an instance of a simpler learning problem. In particular, all of our simpler learning problems we consider can also be written as approachability problems, so our main goal is to understand when a specific approachability instance $(\cP, \cL, \cU)$ is ``reducible'' to a second specific instance $(\cP', \cL', \cU')$.

We define our notion of reducibility as follows. We say there is a \textit{tight reduction} between instance $\cI = (\cP, \cL, \cU)$ and instance $\cI' = (\cP', \cL', \cU')$ if both: i. given any algorithm $\cA$ for $\cI$, we can construct an algorithm $\cA'$ for $\cI'$ with $\AppLoss(\cA') = \AppLoss(\cA)$, and ii. given any algorithm $\cA'$ for $\cI'$, we can construct an algorithm $\cA$ for $\cI$ with $\AppLoss(\cA) = \AppLoss(\cA')$. We say there is a \emph{$c$-approximate reduction} between the two instances if, for sufficiently large $T$, we instead have $\AppLoss(\cA') \leq c\AppLoss(\cA)$ and $\AppLoss(\cA) \leq c\AppLoss_{T}(\cA')$ in the two constraints respectively. Finally, we say that there is a \emph{weak reduction} \emph{from} instance $\cI$ \emph{to} instance $\cI'$ if we only have one direction of the reduction: given an algorithm $\cA'$ for $\cI'$, we can efficiently construct an algorithm $\cA$ for $\cI$ with $\AppLoss(\cA) \leq \AppLoss(\cA')$. Note that a tight reduction implies that $\Rate(\cI') = \Rate(\cI)$, a $c$-approximate reduction implies that $\Rate(\cI)/c \leq \Rate(\cI') \leq c\Rate(\cI)$, and a weak reduction from $\cI$ to $\cI'$ only implies that $\Rate(\cI) \leq \Rate(\cI')$. 

Since any regret minimization instance can be written as an approachability instance, we can use the same notion of reducibility when talking about reductions among instances of regret minimization. The latter sections of this paper will be concerned with understanding when a specific instance of approachability is equivalent to an instance of a specific class of regret minimization problems. There we will use a more restrictive notion of ``linear equivalence'', where all relevant constructions must be provided by affine transformations; we introduce this in Section \ref{sec:linear}.

\subsection{The classical reduction from approachability to external regret minimization}

Finally, we present the classical reduction from approachability to external regret minimization. This is the same reduction in \citep{AbernethyBartlettHazan2011} (and subsequent works), adapted to our notation.

\begin{theorem}\label{thm:classic-reduction}[~\cite{AbernethyBartlettHazan2011}]
Given any instance $\cI = (\cP, \cL, \cU)$ of approachability, there exists a weak reduction from $\cI$ to an instance $\cI' = (\cP', \cL', \cPhi')$ of external regret minimization.
\end{theorem}
\begin{proof}

We will define the target instance $\cI'$ as follows. We will let the new action set $\cP' = \cU$ and the new loss set $\cL' = -(\cP \otimes \cL)$. Note that both $\cP'$ and $\cL'$ are convex subsets of $(\dim \cP)(\dim \cL)$-dimensional Euclidean space, and we can define a bilinear pairing between $\cP'$ and $\cL'$ via $\langle u, -(p \otimes \ell) \rangle = -u(p, \ell)$ for any $u \in \cP'$ and $-(p \otimes \ell) \in \cL'$. Finally, we will let the new benchmark set $\cU'$ consist of all constant functions over $\cP'$ (i.e., for each $x \in \cP'$, there will exist a $u'_{x} \in \cU'$ such that $u'_{x}(p') = x$ for all $p' \in \cP'$). In other words, $\cI'$ is the online linear optimization problem with action set $\cP'$ and loss set $\cL'$ (and is an external regret minimization problem).

Given a low-regret algorithm $\cB$ for solving $\cI'$, we construct the following algorithm $\cA$ for solving $\cI$ using $\cB$ as a black box. In round $t$:

\begin{enumerate}
    \item Set $p'_t = \cB_t(\ell'_1, \ell'_2, \dots, \ell'_{t-1}) \in \cP'$.
    \item Play $p_t \in \cP$ s.t. $\langle p'_t, p_t \otimes \ell \rangle \leq 0$ for all $\ell \in \cL$. Note that such a $p_t$ exists since the instance $(\cP, \cL, \cU)$ is approachable -- in particular, $\langle p'_t, p_t \otimes \ell \rangle = p'_t(p_t, \ell)$ for some $p'_t \in \cU$, and the approachability condition implies that there must exist a $p_t$ where $p'_t(p_t, \ell) \leq 0$ for all $\ell \in \cL$.
    \item Receive the loss $\ell_t$.
    \item Set $\ell'_t = -(p_t \otimes \ell_t) \in \cL'$.
\end{enumerate}
\noindent
By the definition of $\Reg(\bp', \bell')$, we have
\begin{align*}
\Reg(\bp', \bell') 
&= \sum_{t=1}^{T} \langle p'_t, \ell'_t\rangle - \min_{x^* \in \cP'} \sum_{t=1}^{T} \langle x^*, \ell'_t\rangle\\
&= \sum_{t=1}^{T} \langle p'_t, -(p_t \otimes \ell_t)\rangle - \min_{u^* \in \cU} \sum_{t=1}^{T} \langle u^*, -(p_t \otimes \ell_t) \rangle  \\ 
& = \max_{u^* \in \cU} \sum_{t=1}^{T} u^*(p_t, \ell_t) - \sum_{t=1}^{T} p'_t(p_t, \ell_t)\\
& = \AppLoss(\bp, \bell) - \sum_{t=1}^{T} p'_t(p_t, \ell_t) 
\geq  \AppLoss(\bp, \bell).
\end{align*}
This very final inequality follows from the fact that $p'_t(p_t, \ell) \leq 0$ for all $\ell \in \cL$ via our choice of $p_t$ in step 2. Altogether, this analysis implies that $\AppLoss(\cA) \leq \Reg(\cB)$, and therefore provides a weak reduction from $\cI$ to $\cI'$.
\end{proof}

\section{Reducing approachability to (improper) regret minimization}\label{sec:approach-to-regret}

\subsection{The classical reduction is not tight}

Theorem \ref{thm:classic-reduction} proves that the classical reduction of \cite{AbernethyBartlettHazan2011} is a weak reduction from any approachability problem to an external regret minimization problem. It is natural to wonder whether this reduction is in fact tight, or if not, whether the gap in rates between the original approachability problem and the eventual regret problem is small. The following counterexample proves that this is not the case.

\begin{theorem}\label{thm:classic-reduction-not-tight}
    There exists an approachability instance $\cI$ such that $\Rate(\cI) = 0$ but $\Rate(\cI') > 0$, where $\cI'$ is the regret minimization instance obtained by applying the reduction of Theorem \ref{thm:classic-reduction} to $\cI$. In particular, there is no finite $c > 0$ for which the reduction in Theorem \ref{thm:classic-reduction} is a $c$-approximate reduction for all approachability instances, and the reduction is not tight. 
\end{theorem}
\begin{proof}
Consider the following approachability instance $\cI$. Fix any $d' \geq 2$ and let $d = d' + 1$. We will let $\cP = \Delta_d$, $\cL = [0, 1]^d$, and $\cU$ to be the convex hull of the $d'$ bilinear functions $u_i$, where for $i \in [d']$,

$$u_{i}(p, \ell) = \sum_{j=1}^{d'} p_{j}(\ell_{j} - \ell_{i}).$$

\noindent
Here the $u_{i}$ constraint for $i \in [d']$ can be thought of as the regret of moving all probability mass on $p_1$ through $p_{d'}$ to $p_i$. Note, however, that the learner also has the option of placing probability mass on action $d'+1$, which has no approachability constraint associated with this. As a result, $\Rate(\cI) = 0$, since the learner can perfectly approach the negative orthant by always playing $p_t = e_{d'+1}$. 

However, the instance of external regret minimization we reduce to will have a worse rate, as it will require solving a genuine regret minimization problem. Let $\cI' = (\cP', \cL', \cPhi')$ be the regret minimization instance formed by applying the reduction of Theorem \ref{thm:classic-reduction}. This instance has $\cP' = \cU$, $\cL' = -(\cP \otimes \cL)$, and $\cU'$ the set of all constant functions on $\cP'$. We will restrict the loss set further, and insist that the only losses $\ell'_t$ are of the form $\ell'_t = -(U_{d'} \otimes \ell_t)$, where $U_{d'} = (1/d', 1/d', \dots, 1/d', 0) \in \cP$ is the uniform distribution over the first $d'$ coordinates, and $\ell_{t}$ is chosen from the subset $\cL_2 \subseteq \cL$ containing all elements of $\cL$ whose last coordinate equals $0$ (so $\cL_2 \cong [0, 1]^{d'}$). Since this restricts the adversary, it only makes the regret minimization problem easier (and the rate smaller).

The regret of a pair of sequences $\bp'$ and $\bell'$ for this new problem can be written as

\begin{equation}\label{eq:reg1}
    \Reg(\bp', \bell') = \max_{x^{*} \in \cP'} \left(\sum_{t=1}^{T} \langle p'_t, \ell'_t \rangle - \sum_{t=1}^{T} \langle x^{*}, \ell'_t \rangle \right).
\end{equation}

To simplify this further, note that the set $\cP' = \cU$ is given by the convex hull of the $d'$ bilinear functions $u_i$, so we can write each element of $\cP'$ uniquely as a convex combination of these $d'$ functions. For a $p' \in \cP'$, we will write $p'_{i}$ to be the coefficient of $u_i$ in this convex combination (in this way, we identify $\cP'$ with the simplex $\Delta_{d'}$).

Now, for any $p' \in \cP'$ and $\ell' \in \cL'$ (of the above restricted form), we can write

\begin{equation}\label{eq:ip1}
\langle p', \ell'\rangle = -\sum_{i=1}^{d'} p'_{i}\left(\frac{1}{d'}\sum_{j=1}^{d'}(\ell_{j} - \ell_{i})\right) = \left\langle p', \pi(\ell)\right\rangle - \frac{1}{d'} \left(\sum_{i=1}^{d'}\pi(\ell)_{i}\right),\end{equation}

\noindent
where $\pi(\ell)$ is the projection of $\ell$ onto the first $d'$ coordinates. Substituting this in turn into \eqref{eq:reg1} yields
\begin{equation}\label{eq:reg2}
    \Reg(\bp', \bell') = \max_{x^{*} \in \cP'} \left(\sum_{t=1}^{T} \langle p'_t, \pi(\ell_t) \rangle - \sum_{t=1}^{T} \left\langle x^*, \pi(\ell_t) \right\rangle \right).
\end{equation}

Now, note that the adversary can choose $\ell_t$ so that $\pi(\ell_t)$ takes on any value in $\cL_2 = [0, 1]^{d'}$. Similarly, $p'$ can take on any value in $\cP_2 = \cP' = \Delta_{d'}$. Therefore, this problem is at least as hard as the online linear optimization problem with action set $\cP_2 = \Delta_{d_2}$ and loss set $\cL_{2} = [0, 1]^{d_2}$. But this is exactly the online learning with experts problem (with $d'$ experts), which has a regret lower bound of $\Omega(\sqrt{T\log d'})$. It follows that $\Rate(\cI') \geq \Omega(\sqrt{\log d'}) > 0$.
\end{proof}

One may object that the approachability instance in the proof of Theorem \ref{thm:classic-reduction-not-tight} is somewhat degenerate, since the approachability instance has a clear optimal action for the learner (which guarantees perfect approachability). In Appendix \ref{app:non-degenerate} we show that it is possible to slightly perturb this example in a way that avoids the existence of such an optimal action, while maintaining an arbitrarily large gap in rates between the two instances.

\subsection{A tight reduction from approachability to improper $\phi$-regret minimization}

In contrast to Theorem \ref{thm:classic-reduction-not-tight}, it turns out that it is possible to tightly reduce approachability to improper $\phi$-regret. The key observation is that if we ``augment'' the action space of the learner as to also include the benchmark $u \in \cU$ they are competing against each round, we can directly interpret $\AppLoss(\bp, \bell)$ as a specific $\phi$-regret on this space.

\begin{theorem}\label{thm:app-to-improper}
For any approachability instance $\cI = (\cP, \cL, \cU)$, there exists a \emph{tight} reduction from $\cI$ to an improper $\phi$-regret minimization instance $\cI' = (\cP', \cL', \cPhi')$.
\end{theorem}
\begin{proof}
Let $\cX = \cU \otimes \cP$ and $\cY = \cL$. Let $B(x, y): \cX \times \cY \rightarrow \Rset$ be the bilinear form defined via $B(u \otimes p, \ell) = -u(p, \ell)$. Note that we can write 
\begin{align*}
\AppLoss(\bp, \bell) 
& = \max_{u \in \cU}\sum_{t=1}^{T} u(p_t, \ell_t)\\
& = \max_{u \in \cU}\sum_{t=1}^{T} -B(u \otimes p_t, \ell_t) \\
& = \max_{u \in \cU}\left(\sum_{t=1}^{T} B(x_t, \ell_t) - \sum_{t=1}^{T} B(\phi_u(x_t), \ell_t)\right),
\end{align*}

\noindent
where $\phi_u: \cX \rightarrow \cX$ is the linear function defined via $\phi_u(u' \otimes p) = (u' + u) \otimes p$ (for any $u' \in \cU$); note that this is improper as $(u' + u) \otimes p$ is not guaranteed to be within $\cX$. Now, the bilinear form $B$ corresponds to a matrix $M_B$ such that $B(x, y) = \langle x, M_By \rangle$. From this it follows that if we set $\cP' = \cX$, $\cL' = M_B\cY$, and $\cPhi' = \{\phi_{u} \mid u \in \cU\}$, then $\AppLoss(\bp, \bell) = \Reg(\bp', \bell')$ and this is a tight reduction between $\cI$ and $\cI'$. 
\end{proof}

Theorem \ref{thm:app-to-improper} prompts us to question when we can further reduce an improper $\phi$-regret minimization instance to a proper one. We study this question through the lens of linear equivalence in the next section.

\section{Tight linear reductions for regret minimization}\label{sec:linear}

\subsection{Warm-up: weighted regret minimization}\label{sec:weighted}

Theorem \ref{thm:app-to-improper} implies that approachability is equivalent to improper $\phi$-regret minimization. Ideally, we would be able to, in turn, tightly reduce this instance of improper $\phi$-regret minimization to an instance of the better studied problem of external regret minimization (as the classical reduction attempts to do), or at least to an instance of proper $\phi$-regret minimization.

In this section, we examine the question of ``which instances of regret minimization are tightly reducible to each other?'', and specifically ``when is an improper $\phi$-regret instance tightly reducible to a proper $\phi$-regret (or even an external regret) instance?''. However, the space of general reductions is hard to work with directly\footnote{Not only is it hard to work with directly, it quickly becomes meaningless without imposing any constraints on the reductions themselves. Without any such constraints, essentially any two regret-minimization instances $\cI$ and $\cI'$ with $\Rate(\cI)$ and $\Rate(\cI')$ are equivalent. This is because given an algorithm $\cA$ for $\cI$, we can just compute $\AppLoss(\cA)$ and construct an arbitrary algorithm $\cA'$ with $\AppLoss(\cA') = \AppLoss(\cA)$.} -- for this reason, we will restrict ourselves to a subclass of these reductions that we call \emph{linear equivalences}. We will formally define linear equivalences in the next section. For now, we will begin by motivating them via an application where a non-trivial tight reduction is possible: solving the problem of \emph{weighted regret minimization}. 

In the weighted regret minimization problem, we have an action set $\cP$, a loss set $\cL$, and a collection $\Phi = \{\phi_1(p), \phi_2(p), \dots, \phi_N(p)\}$ of proper affine benchmark functions (i.e., each $\phi_i$ maps $\cP$ to itself). Each $\phi_i$ has an accompanying scalar weight $w_i > 0$, representing the importance we place on regret with respect to $\phi_i$. As with the general regret minimization problem and approachability problem, the learner runs a learning algorithm to decide their action at time $t$. The goal of the learner is to minimize the expression:

\begin{equation}\label{eq:weighted-phi-regret}
\WReg(\bp, \bell) = \max_{i \in [N]}\left(w_i \cdot \left[\sum_{t=1}^{T} \langle p, \ell\rangle - \sum_{t=1}^{T} \langle \phi_i(p), \ell \rangle\right]\right).
\end{equation}

\noindent In other words, $\WReg(\bp, \ell)$ simply equals the corresponding $\phi$-regret where the regret with respect to $\phi_i$ is scaled by $w_i$. 

One way to approach a specific weighted regret minimization problem  $\cI$ is to simply ignore the weights and run an algorithm for the corresponding proper $\phi$-regret minimization problem. This would work, in the sense that a sublinear regret algorithm for the unweighted problem would also obtain sublinear regret for the weighted regret minimization problem. But completely ignoring the weights is, of course, lossy: a regret guarantee of $R$ for the unweighted problem only translates to a regret guarantee of $(\max w_i)R$ for the weighted problem. Conversely, if we have a learning algorithm with a regret guarantee of $R$ for the weighted problem, running it on the unweighted instance only translates to a regret guarantee of $R/(\min w_i)$. In the language of our different types of reductions, this naive translation of algorithms is only an $\omega$-approximate reduction, where $\omega = (\max w_i)/(\min w_i)$. 

Despite this, it is possible to exactly write the weighted regret minimization problem as a regret minimization problem of the form described in Section \ref{sec:regret-minimization}. We simply need to observe that:

\begin{equation}\label{eq:rewrite-weighted}
w_i\left(\langle p, \ell\rangle - \langle \phi_i(p), \ell\rangle\right) = \langle p, \ell \rangle - \left\langle w_i\phi_i(p) - (w_i - 1)p, \ell\right\rangle.
\end{equation}

Therefore, if we let $\widetilde{\phi}_i(p) = w_i\phi_i(p) - (w_i - 1)p$ and $\widetilde{\Phi} = \conv(\{\widetilde{\phi}_i\}_{i=1}^{N})$, then the regret minimization instance $(\cP, \cL, \widetilde{\Phi})$ is exactly equivalent to our weighted regret minimization problem. However, this regret minimization problem is in general an \emph{improper} $\phi$-regret minimization problem (even though all the original $\phi_i$ are proper). In fact, this is the case even when the benchmarks $\phi_i$ are all constant and correspond to a weighted external regret minimization problem (e.g., if $\cP = \Delta_3$, $\phi_1(p) = (1, 0, 0)$, and $w_1 = 2$, then $\widetilde{\phi}_1(p) = (2 - p_1, -p_2, -p_3)$).

It turns out that in this case, there is a tight reduction between this improper $\phi$-regret minimization problem and a proper $\phi$-regret minimization problem. 

\begin{theorem}\label{thm:wreg-reduction}
There exists a tight reduction from the improper $\phi$-regret instance $\cI = (\cP, \cL, \widetilde{\cPhi})$ to a proper $\phi$-regret instance $\cI'$.
\end{theorem}
\begin{proof}
We can construct the proper $\phi$-regret minimization problem as follows. Let $W = \max w_i$, and let $\cP' = \cP$, $\cL' = W\cL$, and $\cPhi' = \conv(\{\phi'_i\}_{i=1}^{N})$, where

$$\phi'_i(p) = \frac{w_i}{W}\phi_i(p) + \left(1 - \frac{w_i}{W}\right)p.$$

Note that each $\phi'_i$ is the convex combination of the two proper functions $\phi_i$ and the identity, and is therefore also proper. We let $\cI' = (\cP', \cL', \cPhi')$ denote this proper $\phi$-regret minimization problem.

The reduction from $\cI$ to $\cI'$ is very simple and boils down to rescaling the losses by $W$. The necessary observation is just that:
\begin{equation}\label{eq:example-equivalence}
\langle p - \widetilde{\phi}_i(p), \ell \rangle = \langle p - \phi'_i(p), W\ell\rangle.
\end{equation}

In particular, given an algorithm $\cA'$ for $\cI'$, we can construct an algorithm $\cA$ for $\cI$ with $\Reg(\cA') = \Reg(\cA)$ as follows:

\begin{enumerate}
    \item At the beginning of round $t$, $\cA$ asks $\cA'$ for the $p'_t$ it will play. $\cA$ then plays $p_t = p'_t$.
    \item $\cA$ observes their loss vector $\ell_t$. $\cA$ passes along the loss vector $\ell'_t = W\ell_t$ to $\cA'$.
\end{enumerate}

Equation \eqref{eq:example-equivalence} implies that not only do the worst-case regrets $\Reg(\cA') = \Reg(\cA)$ agree for the above pair of algorithms, but that in any execution of the reduction, both algorithms achieve exactly the same regret. A similar reduction (dividing the loss by $W$ instead of multiplying it by $W$) suffices to show that we can transform an algorithm $\cA$ for $\cI$ to an algorithm $\cA'$ for $\cI'$ with the same regret bound.
\end{proof}

\subsection{Linear equivalences}\label{sec:linear-equivalences}

Motivated by the reduction in Theorem \ref{thm:wreg-reduction}, we introduce the following notion of a ``linear equivalence'' between two regret minimization problems. In short, a linear equivalence will be a tight reduction between two regret minimization instances with the property that the reduction is given entirely by linear transformations.

Formally, let $\cI = (\cP, \cL, \cPhi)$ be a regret minimization instance with $\dim(\cP) = \dim(\cL) = d$. We say that there is a \emph{linear equivalence} between this instance and the instance $\cI' = (\cP, \cL', \cPhi')$ if there exists a bijective correspondence between $\phi \in \cPhi$ and $\phi' \in \cPhi'$ such that the following conditions hold: 

\begin{itemize}
    \item There exist affine maps $S_{\cP', \cP} \in \Aff(\cP', \cP)$ and $S_{\cL, \cL'} \in \Aff(\cL, \cL')$, such that for any $\ell \in \cL$, $p' \in \cP'$, $\ell' = S_{\cL, \cL'}\ell$, and $p = S_{\cP', \cP}p'$, we have:    \begin{equation}\label{eq:Ap_to_A}
    \langle p - \phi(p), \ell\rangle = \langle p' - \phi'(p'), \ell'\rangle.
    \end{equation}

    \item There exist affine maps $S_{\cP, \cP'} \in \Aff(\cP, \cP')$ and $S_{\cL', \cL} \in \Aff(\cL', \cL)$ such that for any $\ell' \in \cL'$, $p \in \cP$, $\ell = S_{\cL', \cL}\ell'$ and $p' = S_{\cP, \cP'}p$, we have:    \begin{equation}\label{eq:A_to_Ap}
    \langle p - \phi(p), \ell\rangle = \langle p' - \phi'(p'), \ell'\rangle.
    \end{equation}
\end{itemize}

Note that the first point above has the following implication. If you have a learning algorithm $\cA'$ for $\cI'$, you can use it to construct a learning algorithm $\cA$ for $\cI$ by following the procedure:

\begin{enumerate}
    \item At the beginning of round $t$, $\cA$ asks $\cA'$ for the $p'_t$ it will play.
    \item $\cA$ then plays $p_t = S_{\cP', \cP}p'_t$.
    \item $\cA$ observes their loss vector $\ell_t$.
    \item $\cA$ passes along the loss vector $\ell'_t = S_{\cL, \cL'}\ell_t$ to $\cA'$.
\end{enumerate}

The guarantee of equation \eqref{eq:Ap_to_A} implies that $\Reg_{T}(\cA) \leq \Reg_{T}(\cA')$. Likewise, the guarantee of equation \eqref{eq:A_to_Ap} (and its surrounding bullet) implies that given an algorithm $\cA$ for $\cI$, we can efficiently construct an algorithm $\cA'$ for $\cI'$ such that $\Reg_{T}(\cA') \leq \Reg_{T}(\cA)$. Together, this shows that a linear equivalence between $\cI$ and $\cI'$ implies a tight reduction between $\cI$ and $\cI'$.

\subsubsection{Simplifying linear equivalences}

We will now introduce a handful of simplifications to the above definition. In particular, note that our linear equivalence in Theorem \ref{thm:wreg-reduction} (for weighted regret) involved only a single invertible linear transformation. We will ultimately show that, under some mild assumptions, it suffices to consider linear equivalences specified by only a single invertible linear transformation.

First, we will assume (possibly by augmenting the sets $\cP$, $\cL$, and $\cL'$ with a fixed extra coordinate) that it is possible to express any affine map over any of these sets as a direct linear transformation, and henceforth restrict our attention to purely linear maps (for both $\phi$ and $S_{\star, \star}$). We will also let $M_{\phi}$ denote the linear transformation $\Id - \phi$; note that this lets us rewrite the regret term $\langle p, \ell\rangle - \langle \phi(p), \ell\rangle$ in the form $\langle M_{\phi}p, \ell \rangle$. This also lets us write equations \eqref{eq:Ap_to_A} and \eqref{eq:A_to_Ap} in the more compact forms
\begin{equation}\label{eq:Ap_to_A_exp}
\langle M_{\phi}S_{\cP', \cP}p', \ell\rangle = \langle M_{\phi'}p', S_{\cL, \cL'}\ell\rangle,
\end{equation}
\noindent and 
\begin{equation}\label{eq:A_to_Ap_exp}
\langle M_{\phi}p, S_{\cL', \cL}\ell'\rangle = \langle M_{\phi'}S_{\cP, \cP'}p, \ell'\rangle.
\end{equation}

We now say that a regret minimization instance $\cI = (\cP, \cL, \cPhi)$ is \emph{minimal} if the following conditions hold:
\begin{itemize}
    \item There does not exist a convex subset $\cP'$ of $\cP$ such that $\Rate(\cP', \cL, \cPhi) = \Rate(\cP, \cL, \cPhi)$.
    \item There does not exist a convex subset $\cL'$ of $\cL$ such that $\Rate(\cP, \cL', \cPhi) = \Rate(\cP, \cL, \cPhi)$.
\end{itemize}

Intuitively, the first constraint captures the property that every (extremal) action in $\cP$ should be useful for the learner; it should be impossible to achieve the same regret bound by only playing a subset of the actions. Similarly, the second constraint captures an analogous property for the adversary -- it should be the case that every extremal loss in $\cL$ is useful for constructing an optimal lower bound.

One advantage of working with minimal regret minimization instances is that reductions between minimal regret minimization instances are specified by invertible linear transformations.

\begin{lemma}\label{lem:minimal}
Let $\cI$ and $\cI'$ be two minimal regret minimization instances. If there is a linear equivalence between $\cI$ and $\cI'$, then the maps $S_{\cP', \cP}$, $S_{\cP, \cP'}$, $S_{\cL, \cL'}$ and $S_{\cL', \cL}$ must all be bijective linear transformations between the two sets.
\end{lemma}
\begin{proof}
We will first show that $S_{\cL, \cL'}$ must be surjective. Consider the regret minimization instance $\cI'_{L} = (\cP', S_{\cL, \cL'}\cL, \cPhi')$, whose loss set $S_{\cL, \cL'}\cL$ contains all losses of the form $S_{\cL, \cL'}\ell$ for $\ell \in \cL$. Note that the reduction from $\cI'$ to $\cI$ described at the end of Section \ref{sec:linear-equivalences} only passes loss vectors in $S_{\cL, \cL'}\cL$ to algorithm $\cA'$, and therefore shows that $\Rate(\cI) \leq \Rate(\cI'_{L})$. Since $S_{\cL, \cL'}\cL \subseteq \cL'$, we in turn have that $\Rate(\cI'_{L}) \leq \Rate(\cI')$. But finally, since there exists a linear equivalence between $\cI$ and $\cI'$, we have $\Rate(\cI) = \Rate(\cI')$, and therefore that $\Rate(\cI'_{S}) = \Rate(\cI')$.

Now, if $S_{\cL, \cL'}\cL$ were a strict subset of $\cL'$, this would violate the assumption that $\cI'$ is minimal. It follows that we must have $S_{\cL, \cL'}\cL = \cL'$, and therefore that $S_{\cL, \cL'}$ is surjective. By symmetry, the linear transformation $S_{\cL', \cL}$ is surjective onto $\cL$. These two facts imply that the dimensions of the convex sets $\cL$ and $\cL'$ must be equal, and the two linear transformations $S_{\cL, \cL'}$ and $S_{\cL', \cL}$ must in fact be bijective linear transformations between $\cL$ and $\cL'$.

Similarly, to show that $\cS_{\cP', \cP}$ is surjective, consider the regret minimization instance $\cI_{P} = (S_{\cP', \cP}\cP', \cL, \cPhi)$. Since in the reduction from $\cI'$ to $\cI$, $\cA$ plays only actions in $S_{\cP', \cP}\cP'$, this reduction actually shows that $\Rate(\cI_{P}) \leq \Rate(\cI')$. But since $S_{\cP', \cP}\cP' \subseteq \cP$, we also have $\Rate(\cI) \leq \Rate(\cI_{P})$. Finally, since $\Rate(\cI) = \Rate(\cI')$ (by the linear equivalence), all three of these rates must be equal. This means that if $S_{\cP', \cP}\cP'$ were a strict subset of $\cP$, this would violate the minimality of $\cP$. Therefore $S_{\cP', \cP}$ is surjective onto $\cP$, $S_{\cP, \cP'}$ is likewise surjective onto $\cP'$, and both transformations must be bijective linear transformations.
\end{proof}

When the transformation matrices are guaranteed to be bijective transformations between the sets (as in Lemma \ref{lem:minimal}), we can specify linear equivalences more succinctly. In particular, we can always without loss of generality take $S_{\cP', \cP} = S_{\cP, \cP'}^{-1}$ and $S_{\cL, \cL'} = S_{\cL', \cL}^{-1}$. The following Lemma shows that for the sake of classifying different regret minimization problems, we can always take $\cP' = \cP$ and $S_{\cP, \cP'} = \Id$ (as we did in our reduction in Section \ref{sec:weighted}).

\begin{lemma}\label{lem:structure}
Let $\cI = (\cP, \cL, \cPhi)$ and $\cI' = (\cP', \cL', \cPhi')$ be two linearly equivalent minimal regret minimization instances, related via $\cP' = S_{\cP, \cP'}\cP$ and $\cL' = S_{\cL, \cL'}\cL$ for invertible linear transformations $S_{\cP, \cP'}$ and $S_{\cL, \cL'}$. Then $\cI$ is also linearly equivalent to the regret minimization instance $\cI'' = (\cP'', \cL'', \cPhi'')$ where:
\begin{eqnarray*}
\cP'' &=& \cP \\
\cL'' &=& S_{\cP, \cP'}^{T}S_{\cL, \cL'}\cL \\
M_{\phi''} &=& S_{\cP, \cP'}^{-1}M_{\phi'}S_{\cP, \cP'}.
\end{eqnarray*}

\noindent
(here in the last line, $\phi''$ is the element of $\cPhi''$ corresponding to $\phi'$ in $\cPhi'$ and $\phi$ in $\cPhi$). Moreover, if $\cI'$ is a proper $\phi$-regret minimization instance, so is $\cI''$; similarly, if $\cI'$ is an external regret minimization instance, so is $\cI''$.

\end{lemma}
\begin{proof}
We can verify that the necessary conditions for a linear equivalence (equations \eqref{eq:Ap_to_A_exp} and \eqref{eq:A_to_Ap_exp}) hold for the above equivalence defined between $\cI$ and $\cI''$. Since all transformations are invertible, it suffices to verify the single equation
$$\langle M_{\phi}p, \ell \rangle = \langle M_{\phi''}S_{\cP, \cP''}p, S_{\cL, \cL''}\ell\rangle.$$

Substituting in $S_{\cP, \cP''} = \Id$, $S_{\cL, \cL''} = S^{T}_{\cP, \cP'}S_{\cL, \cL'}$, and $M_{\phi''} = S_{\cP, \cP'}^{-1}M_{\phi'}S_{\cP, \cP'}$ the RHS of the above expression becomes
$$\langle  S_{\cP, \cP'}^{-1}M_{\phi'}S_{\cP, \cP'}p,S^{T}_{\cP, \cP'}S_{\cL, \cL'}\ell \rangle = \langle M_{\phi'}S_{\cP, \cP'}p, S_{\cL, \cL'}\ell \rangle.$$
The RHS of this expression is in turn equal to $\langle M_{\phi}p, \ell\rangle$ by the guarantee of the original linear equivalence, as desired. To check that the property of being a proper $\phi$-regret minimization instance / external regret minimization instance is preserved, note that from the definition of $M_{\phi''}$ we can read off:
$$\phi''(p) = S_{\cP, \cP'}^{-1}\phi'(S_{\cP, \cP'}p) = S_{\cP, \cP'}^{-1}\phi'(p').$$

\noindent
(defining $p' = S_{\cP, \cP'}p$ in the last line). Now, if $\phi'$ is proper, $\phi'(p') \in \cP'$, and therefore $S^{-1}_{\cP, \cP'}\phi'(p) \in \cP$, and it follows that $\phi''$ is proper. Similarly, if $\phi'(p')$ is constant over $p' \in \cP'$, then $\phi''(p)$ is also constant over $p \in \cP$.
\end{proof}

Inspired by Lemma \ref{lem:structure}, we will restrict ourselves for the remainder of the paper to reductions from $\cI = (\cP, \cL, \cPhi)$ to instances of the form $\cI' = (\cP, (S^{T})^{-1}\cL, \cPhi')$, where $S$ is an invertible linear transformation, where $S_{\cP, \cP'} = \Id$, $S_{\cL', \cL} = S^{T}$, and $S_{\cL, \cL'} = (S^{T})^{-1}$. For such a reduction, the $\phi'$ corresponding to $\phi$ must satisfy

\begin{equation}
\langle M_{\phi}p, S^{T}\ell'\rangle = \langle M_{\phi'}p, \ell'\rangle,
\end{equation}

\noindent
for all $\ell' \in \cL'$ and $p \in \cP$ (this follows from substituting in our specific transformations into \eqref{eq:Ap_to_A_exp}). We can rewrite this in the form:
$$\langle (SM_{\phi} - M_{\phi'})p, \ell'\rangle = 0$$

\noindent
for all $p \in \cP$ and $\ell' \in \cL'$. Since we assume that $\vecspan(\cP) = \vecspan(\cL') = \Rset^d$, this implies that

\begin{equation}
M_{\phi'} = SM_{\phi}
\end{equation}

\noindent
and in particular that

\begin{equation}\label{eq:phi_to_phip}
\phi'(p) = p + S(\phi(p) - p).
\end{equation}

\subsection{Reducing improper $\phi$-regret to external regret}

In this section we will provide a complete characterization of when (minimal) regret minimization instances are linearly equivalent to external regret minimization instances (recall that these are instances where all of the functions $\phi \in \cPhi$ are constant over $\cP$). 

Given a set of points $U$, let $\vecspan_{\Aff}(U)$ denote the \emph{affine span} of $U$, that is, the affine subspace formed by all vectors of the form $\sum_{i=1}^{k}\lambda_{i}v_i$ for any $k > 0$, $v_i \in U$, and $\lambda_{i} \in \Rset$ such that $\sum_{i} \lambda_{i} = 1$ (note that the $\lambda_i$ need not be non-negative). We prove the following characterization.

\begin{theorem}\label{thm:external-char}
Let $\cI = (\cP, \cL, \cPhi)$ be a regret minimization instance. Then $\cI$ is linearly equivalent to an external regret minimization instance if and only if the following conditions are met: 

\begin{enumerate}
    \item For all $\phi \in \cPhi$, each $\phi$ has a \emph{single} fixed point $p$ in $\vecspan_{\Aff}(\cP)$.
    \item For all $\phi_1, \phi_2 \in \cPhi$, $\phi_1(p) - \phi_2(p)$ is a constant for all $p \in \cP$. 
\end{enumerate}
\end{theorem}

Note that by the assumption that the instance $\cI$ is valid, every $\phi \in \cPhi$ must contain a fixed point in $\cP$; condition 1 above rules out the existence of any other fixed points in the entire affine subspace containing $\cP$.

Before we proceed to prove Theorem \ref{thm:external-char}, we consider a few illustrative examples:

\begin{itemize}
    \item \textbf{Classic external regret}: In the classic external-regret setting of learning with experts, $\cP = \Delta_{d}$, $\cL = [0, 1]^{d}$, and the set $\cPhi$ contains all constant functions $\phi: \cP \rightarrow \cP$. If $\phi$ is the constant map $\phi(p) = p_{\phi}$ (for some fixed $p_{\phi} \in \cP$), then $\phi$ has the unique fixed point $p_{\phi}$. It is also clear that the difference of any two constant functions is constant.
    \item \textbf{Improper regret that can be reduced to external regret}: Consider the setting where $\cP = \Delta_{2}$, $\cL = [0, 1]^2$, and $\cPhi$ contains all linear functions of the form $\phi_{\alpha}(x, y) = (x, (2-\alpha)y -\alpha x)$ for all $\alpha \in [0, 1]$. Note that this is an improper $\phi$-regret minimization instance since $\phi_{\alpha}$ often sends points in the simplex to points outside the simplex (e.g., $\phi_{1/2}(2/3, 1/3) = (2/3, 1/6)$). Nonetheless, $\phi_{\alpha}$ satisfies the constraints of Theorem \ref{thm:external-char}: $(1-\alpha, \alpha)$ is a fixed point of $\phi_{\alpha}$, and $\phi_{\alpha_1}(p, 1-p) - \phi_{\alpha_2}(p, 1-p) = (0, \alpha_2 - \alpha_1)$ for any $(p, 1-p) \in \Delta_2$. And indeed, one can check that under the invertible transformation matrix    
$$S = \begin{pmatrix} 
0 & 1 \\
1 & -1
\end{pmatrix}$$
$\phi_{\alpha}(x,  y)$ gets transformed (via \eqref{eq:phi_to_phip}) to $\phi'_{\alpha}(x, y) = ((1-\alpha)(x+y), \alpha(x+y))$ (which takes on the constant value $(1-\alpha, \alpha)$ for $(x, y) \in \Delta_2$).

\item \textbf{Swap regret}: Swap regret minimization is a proper $\phi$-regret minimization where $\cP = \Delta_d$, $\cL = [0,1]^d$, and $\cPhi$ is the convex hull of all $\phi_{\pi}(p, \ell) = \sum_{i=1}^{d}p_{i}(\ell_{i} - \ell_{\pi(i)})$ where $\pi$ ranges over all ``swap functions'' from $[d]$ to $[d]$. It can easily be checked that these functions do not differ by a constant (for $d > 1$)\footnote{The case $d=2$ is interesting. When $d=2$ it is possible to bound swap regret by at most twice external regret, and so sublinear external regret implies sublinear swap regret. But there is still no linear equivalence between swap regret and external regret in this case, just as there is none between weighted external regret and external regret in the next example.} and therefore by Theorem \ref{thm:external-char} there is no linear equivalence between swap regret and external regret.

\item \textbf{Weighted external regret}: Consider the weighted regret minimization setting of Section \ref{sec:weighted} where each original function $\phi_i(p)$ is a constant function (so this captures a weighted version of external regret minimization). The expression \eqref{eq:rewrite-weighted} describes how to express this in terms of an improper $\phi$-regret minimization problem with functions $\widetilde{\phi_i}(p)$. Looking at the expressions for $\widetilde{\phi_i}(p)$, we can see that a weighted external regret minimization instance is linearly equivalent to an ordinary external regret instance if and only if all the weights $w_i$ are equal.
\end{itemize}

\begin{figure}[t]
\centering
\includegraphics[scale=0.2]{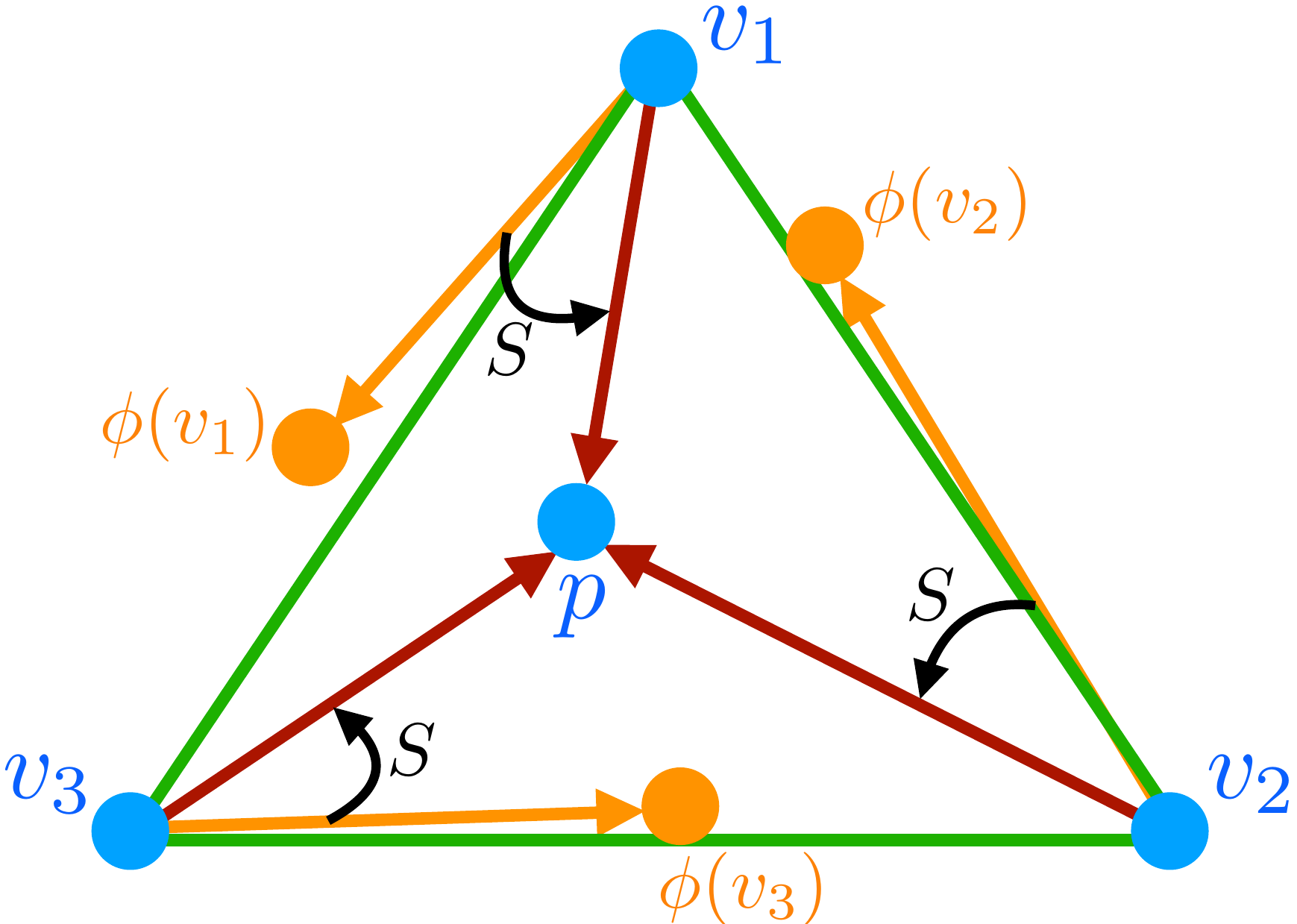}
\caption{Illustration of the definition of $S$ in
the proof of Lemma~\ref{lemma:one-phi-external}. $v_1, v_2, v_3$ represent the vertices of the simplex in $\Rset^3$ and $p$ is the fixed point of $\phi$. The vector $\phi(v_k)$, $k \in [3]$, may not be in the simplex. The linear function $S$ maps vector $(\phi - I)(v_k)$ into $p - v_k$.}
\label{fig:true-detective}
\end{figure}

To prove Theorem \ref{thm:external-char}, we will use the following lemma (which proves the above characterization in the case of a single transformation $\phi$).

\begin{lemma}
\label{lemma:one-phi-external}
Let $\phi: \Rset^{d} \rightarrow \Rset^{d}$ be a linear function, and let $\cP \subseteq \Rset^{d}$ be a convex set. Then the following two conditions are equivalent:

\begin{itemize}
    \item There exists an invertible linear transformation $S$ such that $\phi'(p) = p + S(\phi(p) - p)$ is constant for all $p \in \cP$.
    \item $\phi(p)$ has a single fixed point $p_{\phi}$ in $\vecspan_{\Aff}(\cP)$. 
\end{itemize}
\end{lemma}

\begin{proof}
First, assume there does exist such a linear transformation $S$. Assume to the contrary that $\phi$ has two distinct fixed points $p_1, p_2 \in \vecspan_{\Aff}(\cP)$. Note that $\phi'(p_1) = p_1 + S(\phi(p_1) - p_1) = p_1$ so $p_1$ is a fixed point of $\phi'$; likewise $p_2$ is a fixed point of $\phi'$. But if $\phi'$ is constant on $\cP$, it must (as a linear transformation) must also be constant on $\vecspan_{\Aff}(\cP)$, implying that we must have $\phi'(p_1) = \phi'(p_2)$, a contradiction.

Now, assume $\phi(p)$ contains a single fixed point $p_{\phi} \in \vecspan_{\Aff}(\cP)$. If $\dim(\vecspan_{\Aff}(\cP)) = k$, pick any $k$ points $v_1, v_2, \dots, v_k \in \cP$ that together with $p_{\phi}$ affinely span $\cP$. In particular, the $k$ vectors $(v_i - p_{\phi})$ are linearly independent. 

For each $i \in [k]$, let $w_i = \phi(v_i) - v_i$. We claim that the $k$ vectors $w_i$ are also linearly independent. If they were not, there would exist a non-trivial linear combination $\sum_{i}\lambda_i w_i$ equal to $0$. Expanding this out (and using the fact that $\phi$ is linear), we have 
$$\phi\left(\sum_{i}\lambda_{i}v_i\right) = \sum_{i}\lambda_{i}v_i.$$

Let $L = \sum_{i}\lambda_{i}$. Adding $(1-L)p_{\phi}$ to both sides (and using the fact that $\phi(p_{\phi}) = p_{\phi}$) we have
$$\phi\left((1-L)p_{\phi} + \sum_{i}\lambda_{i}v_i\right) = (1-L)p_{\phi} + \sum_{i}\lambda_{i}v_i.$$

But the element $v = (1-L)p_{\phi} + \sum_{i}\lambda_{i}v_i$ belongs to $\vecspan_{\Aff}(\cP)$, so (by assumption) the only way it could be a fixed point of $\phi$ is if $v = p_{\phi}$. But $v - p_{\phi} = \sum_{i}\lambda_{i}(v_i - p)$, so this would imply that the original vectors $p - v_i$ were not linearly independent, a contradiction.

We can now directly construct our transformation $S$. We will choose any invertible $S$ that satisfies
$$S(\phi(v_i) - v_i) = p_{\phi} - v_i$$

\noindent
for all $i \in [k]$. Since both sets of vectors $w_i = \phi(v_i) - v_i$ and $p_{\phi} - v_i$ are linearly independent, we can choose an invertible such $S$ (technically we also need to specify the action of $S$ on vectors not contained in the span of $w_i$, but for these we can choose an arbitary invertible transformation between the remaining sets of basis vectors). 

We now claim that the resulting $\phi'(p) = p + S(\phi(p) - p)$ is a constant function (and in fact, equals $p_{\phi}$ for all $p \in \cP$). To see this, write any $p$ in the form $p = p_{\phi} + \sum_{i=1}^{k}\lambda_{i}(v_i - p_{\phi})$ for some $\lambda_{i}  \in \Rset$. Then
$$S(\phi(p) - p) = \sum_{i=1}^{k}\lambda_i S(\phi(v_i) - v_i) = \sum_{i=1}^{k}\lambda_{i}(p_{\phi} - v_i) = -\sum_{i=1}^{k}\lambda_{i}(v_i - p_{\phi}),$$

\noindent
and therefore $p + S(\phi(p) - p) = p_{\phi}$, as desired.
\end{proof}

We can now complete the proof of Theorem \ref{thm:external-char}.

\begin{proof}[Proof of Theorem~\ref{thm:external-char}]
We first prove that both conditions are necessary. The necessity of the first condition follows immediately from Lemma~\ref{lemma:one-phi-external}. To see that the second condition is necessary, consider any $\phi_1, \phi_2 \in \cPhi$. By \eqref{eq:phi_to_phip} we have
\begin{align*}
\phi'_1(p) &= p + S(\phi_1(p) - p) \\
\phi'_2(p) &= p + S(\phi_2(p) - p)
\end{align*}

\noindent
and therefore
$$\phi'_1(p) - \phi'_2(p) = S(\phi_1(p) - \phi_2(p)).$$

If $S$ is invertible and $\phi'_1(p) - \phi'_2(p)$ is a constant for all $p \in \cP$, then $\phi_1(p) - \phi_2(p)$ must also be a constant for all $p \in \cP$.

To show that these two conditions are sufficient, single out a representative $\psi \in \cPhi$. By Lemma~\ref{lemma:one-phi-external} we can construct an $S$ such that the corresponding $\psi'$ satisfies
$$\psi'(p) = p + S(\psi(p) - p) = p_{\psi}$$

\noindent 
for all $p \in \cP$. Now, since any pair of functions in $\cPhi$ differ by a constant, for each $\phi \in \cPhi$, write $\phi(p) = \psi(p) + \delta_{\phi}$ where $\delta_{\psi}$ is some $p$-independent constant in $\Rset^{d}$. If we let $p_{\phi}$ denote the fixed point of $\phi$ in $\cP$, then since $\phi(p_{\phi}) = p_{\phi}$, so we can deduce that $\delta_{\phi} = p_{\phi} - \psi(p_{\phi})$. Note then that
\begin{eqnarray*}
\phi'(p) &=& p + S(\phi(p) - p) \\
&=& p + S(\psi(p) + \delta_{\phi} - p) \\
&=& p + S(\psi(p) - p + p_{\phi} - \psi(p_{\phi})) \\
&=& p + (p_{\psi} - p) - (p_{\psi} - p_{\phi}) \\
&=& p_{\phi},
\end{eqnarray*}

\noindent
so each $\phi'(p)$ is constant for all $p \in \cP$.
\end{proof}

\subsection{Reducing improper $\phi$-regret to proper $\phi$-regret}

Theorem \ref{thm:external-char} shows that the property of being equivalent to an external regret minimization problem is rather stringent: only very structured improper $\phi$-regret minimization problems (and hence, approachability instances) can be reduced via linear transformations to such instances.

However, the class of proper $\phi$-regret minimization problems is far broader than the class of external regret minimization problems. Notably, every regret minimization problem we have considered thus far either is or is linearly equivalent to a proper $\phi$-regret minimization instance. This raises the natural question of whether every improper $\phi$-regret minimization instance can be linearly transformed into a proper one.

In this section we show that the answer to this question is no. In Section \ref{sec:counterexamples} we give some examples of ``atypical'' improper $\phi$-regret minimization problems that we prove are not linearly equivalent to any proper $\phi$-regret minimization problems.

Given this, we can also ask whether (in the same vein as Theorem \ref{thm:external-char}) we can cleanly characterize the set of regret-minimization problems which are equivalent to proper $\phi$-regret minimization. In Section \ref{sec:algorithmic-characterization} we give an algorithmic procedure for deciding whether a given regret-minimization instance can be reduced to a proper one in the case where the sets $\cP$, $\cL$, and $\cPhi$ are all polytopes.

\subsubsection{Irreducible improper instances}\label{sec:counterexamples}

In this section we provide two (classes of) examples of improper $\phi$-regret minimization problems that cannot be linearly reduced to a proper $\phi$-regret minimization problem. These two examples will each illustrate different obstructions to the property of being reducible.

Throughout this section, we will assume that $\cP = \Delta_{d}$ and $\cL = [0, 1]^{d}$, and only vary the set $\cPhi$. Before we introduce our first class of examples, we will prove the following lemma, which shows that if the instance $\cI = (\cP, \cL, \cPhi)$ is linearly equivalent to a proper $\phi$-regret minimization problem, then all transformations in $\cPhi$ must share a left eigenvector of eigenvalue $1$.

\begin{lemma}\label{lem:left-eigen}
 Assume $\cP$ is contained within an affine subspace of $\Rset^d$. If $(\cP, \cL, \cPhi)$ is linearly equivalent to an instance of proper $\phi$-regret minimization, then there exists a non-zero $v \in \Rset^d$ such that $\langle \phi(p), v \rangle = \langle p, v \rangle$ for any $\phi \in \cPhi$ and $p \in \cP$.
\end{lemma}
\begin{proof}
By \eqref{eq:phi_to_phip}, this linear equivalence sends $\phi(p)$ to the function $\phi'(p) = p + S(\phi(p) - p)$, for some invertible linear transformation $S$. Since the target regret minimization instance is proper, $\phi'(p)$ is guaranteed to lie in $\cP$.

Now, since $\cP$ is contained within an affine subspace of $\Rset^{d}$, there exists some $w \in \Rset^{d}$ and $b \in \Rset$ such that $\langle p, w \rangle = b$ for all $p \in \cP$. This implies that
$$\langle \phi'(p), w\rangle = \langle p, w\rangle + \langle S(\phi(p) - p), w \rangle.$$

Since $\langle \phi'(p), w\rangle = \langle p, w \rangle = b$, this reduces to $\langle S(\phi(p) - p), w\rangle = 0$, which in turn can be written as $\langle \phi(p) - p, S^{T}w\rangle = 0$. Therefore the vector $S^{T}w$ satisfies the constraints for $v$ in the lemma statement.
\end{proof}

Note that since the simplex $\Delta_d$ is contained within the affine hyperplane $\sum_{i} p_i = 1$, Lemma~\ref{lem:left-eigen} applies for all examples in this section. Also, we can alternatively think of the statement of Lemma \ref{lem:left-eigen} as stating that the matrices $M_{\phi}$ must all share a non-trivial left kernel element (an element $v$ such that $v^{T}M_{\phi} = 0$). To construct an irreducible instance, it suffices to find examples of convex sets of matrices that do not all share the same left kernel element. 

Constructing such a set is made more difficult by the fact that each $\phi \in \cPhi$ must have a fixed point in $\cP$, which translates to the fact that every matrix $M_{\phi}$ must have a non-trivial right kernel element $p_{\phi}$ whose entries are all non-negative (any such element can be scaled to lie in $\cP = \Delta_d$). We would also like the different $M_{\phi}$ to not all share the same right kernel element (because in that case each $\phi \in \cPhi$ would have the same fixed point, and the instance would not be minimal). 

We can construct such a set by noticing that the set of skew-symmetric matrices of odd dimension is a linear subspace of the set of matrices. In particular, if for some $a, b, c \in \Rset_{\geq 0}$ we consider the skew-symmetric matrix
$$M_{\phi_{a,b,c}} = \begin{pmatrix}
0 & a & -b \\
-a & 0 & c \\
b & -c & 0
\end{pmatrix},$$
\noindent
then the vector $v = (c, b, a)$ belongs to both the left-kernel and right-kernel of $M_{\phi_{a,b,c}}$. In particular, this means that if we take $\cPhi = \conv(\phi_{1,0,0}, \phi_{0,1,0}, \phi_{0,0,1})$ where:
\begin{align*}
\phi_{1,0,0}(p_1, p_2, p_3) &= (p_1 - p_2,\, p_1 + p_2,\, p_3) \\
\phi_{0,1,0}(p_1, p_2, p_3) &= (p_1 + p_3,\, p_2,\, p_3 - p_1) \\
\phi_{0,0,1}(p_1, p_2, p_3) &= (p_1,\, p_2 - p_3,\, p_2 + p_3),
\end{align*}

\noindent
then $\cI = (\cP, \cL, \cPhi)$ is a valid improper $\phi$-regret minimization instance which, by Lemma \ref{lem:left-eigen}, is \emph{not} linearly equivalent to a proper $\phi$-regret minimization instance.

One might wonder whether Lemma \ref{lem:left-eigen} is the sole obstacle to linear equivalence to proper $\phi$-regret minimization. Interestingly, this is not the case. We now give a second $\cPhi$ that does not violate Lemma~\ref{lem:left-eigen} (i.e., all $M_{\phi}$ share a left-kernel element), but where the instance $(\cP, \cL, \cPhi)$ is still provably not linearly equivalent to proper $\phi$-regret minimization. To do so, we will rely on the following obstruction.

\begin{lemma}\label{lem:no-neg}
Consider an improper $\phi$-regret minimization instance $\cI = (\cP, \cL, \cPhi)$. If there exists an extreme point $x$ of $\cP$ and $\phi_1, \phi_2 \in \cPhi$ with the property that
$$\phi_2(x) - x = -\alpha(\phi_1(x) - x) \neq 0$$

\noindent
for some real $\alpha > 0$, then $\cI$ is not linearly equivalent to a proper $\phi$-regret minimization instance.
\end{lemma}
\begin{proof}
If such a linear equivalence existed, by \eqref{eq:phi_to_phip}, we would have that $\phi'_1(x) = x + S(\phi_1(x) - x)$ and $\phi'_2(x) = x + S(\phi_2(x) - x) = x - \alpha S(\phi_1(x) - x)$. But since $x$ is an extreme point of $\cP$, it is impossible for both $x + w$ and $x - \alpha w$ to both lie in $\cP$ for a non-zero vector $w$.
\end{proof}

We now present a valid improper $\phi$-regret minimization instance (found by computer search) where Lemma \ref{lem:no-neg} applies but Lemma \ref{lem:left-eigen} does not. This example is parametrized by the two matrices:
$$A = \begin{pmatrix}
-6 & 8 & -9 \\
2 & -1 & -9 \\
4 & -7 & 18
\end{pmatrix},
B = \begin{pmatrix}
10 & -3 & -7 \\
6 & -6 & 10 \\
-16 & 9 & -3
\end{pmatrix}.$$

\noindent
Note that both $A$ and $B$ share the same left-kernel element $(1, 1, 1)$ (so they both map the hyperplane containing $\cP$ to itself). We will consider the set of regret functions $\cPhi = \{\Id + \gamma_{a}A + \gamma_{b}B \mid \gamma_{a} \in [-1, 1], \gamma_{b} \in [-1, 1]\}$. One can also verify computationally that all elements of $\cPhi$ contain a fixed point within the simplex $\Delta_{3}$ and that this fixed point is not static and changes depending on $\gamma_{b}/\gamma_{a}$ (see Figure \ref{fig:eigenvector_plot}). 

\begin{figure}
    \centering
    \includegraphics[width=0.4\textwidth]{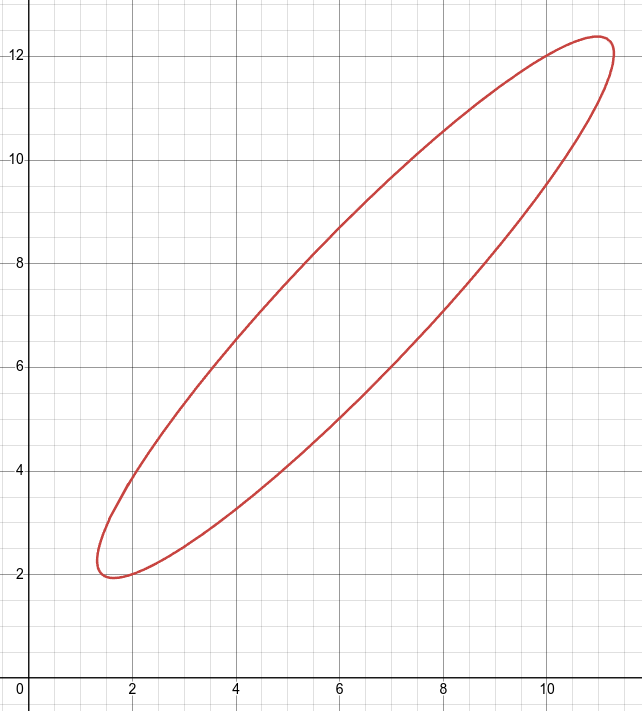}
    \caption{For any $t$, the matrix $M(t) = A + tB$ has a right kernel element of the form $(x(t), y(t), 1)$, where $x(t) = \frac{72t^2-46t + 81}{2(21t^2 + 8t + 5)}$ and $y(t) = \frac{72t^2 - 41t + 36}{21t^2 + 8t + 5}$. The above diagram plots $(x(t), y(t))$ for all $t \in \Rset$, and shows that this right kernel element always has non-negative entries (and hence has a multiple that lies in $\Delta_3$).}
    
    \label{fig:eigenvector_plot}
\end{figure}

But now, consider the two functions $\phi_{1}(p) = p + Ap$ and $\phi_{2}(p) = p - Ap$ both belonging to $\cPhi$. For any $x \in \cP$ it is the case that $\phi_1(x) - x = Ax = - (\phi_2(x) - x)$. Choosing $x$ to b e the extreme point $(1, 0, 0)$, $Ax$ is non-zero and thus by Lemma \ref{lem:left-eigen} this instance is not equivalent to a proper $\phi$-regret minimization instance.

\subsubsection{An algorithmic characterization}\label{sec:algorithmic-characterization}

Finally, we will show how to algorithmically decide whether a given regret minimization instance $\cI = (\cP, \cL, \cI)$ is linearly equivalent to a proper $\phi$-regret minimization problem via solving an appropriate linear program for the transformation $S$. For simplicity, we will only handle the case where $\cP$ and $\cI$ are both polytopes, with $\cP$ being the convex hull of the $N$ vertices $p_1, p_2, \dots, p_N$ and $\cI$ being the convex hull of the $M$ regret functions $\phi_{1}, \phi_{2}, \dots, \phi_{M}$.

Recall that (by \eqref{eq:phi_to_phip}) we have $\phi'_{i}(p) = p + S(\phi(p) - p)$. Note that $\phi'$ is a proper $\phi$-function if and only if $\phi'_i(p_j)$ lies within $\cP$ for all vertices $p_j$ (by convexity, this implies that $\phi'_i(p) \in \cP$ for any other $p \in \cP$). So there exists a reduction iff there exists an invertible $S$ such that for every $i \in [M]$ and $j \in [N]$, we have the constraint
\begin{equation} \label{eq:s-constraint}
p_j + S(\phi_i(p_j) - p_j) \in \cP.
\end{equation}

We can in turn rephrase \eqref{eq:s-constraint} as several linear constraints on the entries of the matrix $S$. In particular, if we introduce the auxiliary variables $\lambda_{i, j, k}$ (for $k \in [N]$) we can rewrite the set of constraints expressed by \eqref{eq:s-constraint} as the following linear program (whose variables are the $d^2$ entries of $S$ and the $\lambda_{i,j,k}$):
\begin{eqnarray*}
p_{j} + S(\phi_{i}(p_j) - p_j) &=& \sum_{k=1}^{N}\lambda_{i, j, k}p_{k} \text{ for all } i \in [M], j \in [N] \\
\sum_{k=1}^{N} \lambda_{i, j, k} &=& 1 \text{ for all } i \in [M], j \in [N] \\
\lambda_{i, j, k} &\geq& 0 \text{ for all } i \in [M], j, k \in [N]
\end{eqnarray*}

The above constraints define a convex cone of possible values of $S$ which we denote by $\cS$. For example, this cone always contains the solution $S = 0$; we would like to now decide whether it contains any invertible matrices. Luckily, this is straightforward to do in a randomized manner by the following lemma.

\begin{lemma}
Given any convex set $\cS \subseteq \Rset^{d \times d}$ of $d$-by-$d$ matrices, either every matrix in $\vecspan(\cS)$ is non-invertible, or almost all\footnote{All but a measure zero subset in the Euclidean measure of $\vecspan(\cS)$.} matrices in $\vecspan(\cS)$ are invertible.
\end{lemma}
\begin{proof}
Given any linear subspace of matrices $V \subseteq \Rset^{d \times d}$, the constraint $\det(M) = 0$ defines an algebraic variety on this space. Any algebraic variety in a Euclidean space either has measure zero or is equal to the entire space.
\end{proof}

Given this, it suffices to simply test whether a random element in $\vecspan(\cS)$ is invertible. We can efficiently generate a basis of $\vecspan(\cS)$ by repeatedly solving the linear program above (finding the extreme points in $\cS$ in a direction orthogonal to all basis elements found so far), and once we have such a basis, we can test a random linear combination of the basis elements for invertibility. 

This entire procedure takes time polynomial in $N$, $M$, and $d$. It can also likely be extended beyond polytopes to any case where have a membership oracle for $\cS$ (i.e., a procedure that decides whether \eqref{eq:s-constraint} is satisfied for all $\phi \in \cPhi$ and $p \in \cP$), modulo numerical issues with testing invertibility. Note also that this procedure not only checks whether a given instance is reducible, by also provides a valid transformation $S$ in the case that it is.
\section{Conclusion}\label{sec:conclusion}

This work uncovers a nuanced relationship between Blackwell's
approachability and no-regret learning. While fundamentally
equivalent, this equivalence does not always preserve optimal
convergence rates, highlighting the importance of understanding the
fine-grained details of reductions between these problems. Our
introduction of improper $\phi$-regret minimization provides a useful
tool for bridging this gap, allowing for tight reductions from any
approachability instance to a generalized regret minimization
problem. The existence of improper $\phi$-regret instances that cannot
be reduced to standard classes suggests that approachability covers a
broader range of problems than can be easily expressed in traditional
online learning terms.  This also motivates a more detailed study of
the complex relationship between approachability and improper
$\phi$-regret, potentially leading to deeper connections and new
algorithmic solutions.

\bibliographystyle{abbrvnat}
\bibliography{references}

\newpage
\appendix

\section{Approaching the negative orthant}\label{app:orthant}

In this appendix we show that any approachability problem can equivalently be written in terms of approaching the negative orthant.

Traditionally (as in \cite{Blackwell1956}), instead of specifying a constraint set $\cU$, an approachability problem specifies a vector-valued bilinear payoff function $\bu\colon \cP \times \cL \rightarrow \Rset^k$ along with a set $\cS \subseteq \Rset^k$ that the learner would like to approach. In particular, they would like to minimize the average distance

$$\AppDist(\bp, \bell) = \mathrm{dist}_{\nu}\left(\frac{1}{T}\sum_{t=1}^{T}\bu(p_t, \ell_t), \cS\right),$$

\noindent
where $\mathrm{dist}(x, \cS)$ is the minimum distance between $x$ and the set $\cS$ under some norm $\nu$. We prove the following result:

\begin{theorem}\label{thm:dist-to-loss}
For any norm $\nu$, set $\cS$, and payoff $\bu$, there exists a corresponding convex set $\cU \subseteq \Aff(\cP\otimes \cL, \Rset)$ such that

$$\AppDist(\bp, \bell) = \AppLoss(\bp, \bell).$$

\noindent
where here $\AppLoss$ is understood to mean the approachability loss \eqref{eq:app_loss_def} with respect to the set of constraints $\cU$.
\end{theorem}
\begin{proof}
For any $r$, consider the set of points $\cS_r$ within distance $r$ (under $\nu$) of $\cS$. We can write $\cS_r$ as $\cS + rB$, where $B$ is the unit ball in the norm $\nu$. We can then write $\cS_r$ as the intersection of the extremal halfspaces in all directions, where the extremal halfspace in direction $v$ (for any $v \in \mathbb{S}^{d-1}$ on the unit sphere) can be written in the form:

$$\{x \in \Rset^{k} \mid \langle v, x \rangle \leq a_{v} + b_{v}r\}$$

\noindent
for some constants $a_v$ and $b_v$ (with $b_v \geq 0$). It follows that for any $x \in \Rset^{k}$, we can write

$$\mathrm{dist}_{\nu}(x, \cS) = \max_{v \in \mathbb{S}^{d-1}}\left(\frac{\langle v, x\rangle - a_v}{b_v}\right).$$

\noindent
If we choose our set $\cU$ to contain all functions $u_{v}$ of the form

$$u_{v}(p, \ell) = \frac{\langle v, \bu(p, \ell)\rangle - a_v}{b_v},$$

\noindent
(for all $v \in \mathbb{S}^{d-1}$), then it follows that $\AppDist(\bp, \bell) = \AppLoss(\bp, \bell)$.
\end{proof}

\section{Non-degenerate counterexample to the reduction}\label{app:non-degenerate}

One may object that the approachability instance $\cI$ presented in Theorem~\ref{thm:classic-reduction-not-tight} with $\Rate(\cI) = 0$ is somewhat degenerate, as there is a single action by the learner that perfectly approaches the negative orthant. In this appendix we show that this is easily addressed; we can use a similar construction to obtain non-degenerate approachability instances where $\Rate(\cI')/\Rate(\cI)$ is arbitrarily large. The main idea is to embed a non-trivial (but very easy) regret-minimization problem on the ``unused'' coordinates of the instance.

\begin{theorem}\label{thm:classic-reduction-not-tight-nondegenerate}
    For any $c > 0$, we can construct an approachability instance $\cI$ such that $\Rate(\cI') > c \cdot \Rate(\cI)$, where $\cI'$ is the regret minimization instance obtained by applying the reduction of Theorem \ref{thm:classic-reduction} to $\cI$. In particular, there is no finite $c > 0$ for which the reduction in Theorem \ref{thm:classic-reduction} is a $c$-approximate reduction for all approachability instances, and the reduction is not tight. 
\end{theorem}
\begin{proof}
Consider the following approachability instance $\cI$. Fix any $\eps > 0$, $d_1, d_2 > 1$ and let $d = d_1 + d_2$. We will let $\cP = \Delta_d$, $\cL = [0, 1]^d$, and $\cU$ to be the convex hull of the $d$ bilinear functions $u_i$, where for $i \in [d_1]$,

$$u_{i}(p, \ell) = \eps\sum_{j=1}^{d_1} p_j(\ell_j - \ell_i),$$

\noindent
and for $i \in [d_2]$,

$$u_{d_1+i}(p, \ell) = \sum_{j=d_1+1}^{d_1+d_2} p_{j}(\ell_{j} - \ell_{d_1 + i}).$$

\noindent
Here the $u_{i}$ constraint for $i \in [d_1]$ can be thought of as the regret of moving all probability mass on $p_1$ through $p_{d_1}$ to $p_i$ (but weighted by $\eps$); similarly, the $u_{d_1 + i}$ constraint can be thought of as the regret of moving all probability mass on $p_{d_1+1}$ through $p_d$ to $p_{d_1 + i}$.

We will first show that $\Rate(\cI) = O(\eps\sqrt{\log d_1})$. To see this, note that if we ignore the last $d_2$ coordinates (always assigning weight $0$ to them) and run a standard online-learning algorithm over the simplex $\Delta_{d_1}$ (e.g., Hedge) it is possible to construct an approachability algorithm $\cA$ with $\AppLoss_{T}(\cA) = O(\sqrt{T\log d_1})$.

However, the instance of external regret minimization we reduce to will have a worse rate. Let $\cI' = (\cP', \cL', \cPhi')$ be the regret minimization instance formed by applying the reduction of Theorem \ref{thm:classic-reduction}. This instance has $\cP' = \cU$, $\cL' = -(\cP \otimes \cL)$, and $\cU'$ the set of all constant functions on $\cP'$. We will restrict the loss set further, and insist that the only losses $\ell'_t$ are of the form $\ell'_t = -(U_{d_2} \otimes \ell_t)$, where $U_{d_2} = (0, 0, \dots, 0, 1/d_2, 1/d_2, \dots, 1/d_2) \in \cP$ is the uniform distribution over the last $d_2$ coordinates, and $\ell_{t}$ is chosen from the subset $\cL_2 \subseteq \cL$ contains all elements of $\cL$ whose first $d_1$ coordinates equal $0$ (so $\cL_2 \cong [0, 1]^{d_2}$). Since this restricts the adversary, it only makes the regret minimization problem easier (and the rate smaller).

The regret of a pair of sequences $\bp'$ and $\bell'$ for this new problem can be written as

\begin{equation}\label{eq:reg1app}
    \Reg(\bp', \bell') = \max_{x^{*} \in \cP'} \left(\sum_{t=1}^{T} \langle p'_t, \ell'_t \rangle - \sum_{t=1}^{T} \langle x^{*}, \ell'_t \rangle \right).
\end{equation}

To simplify this further, note that the set $\cP' = \cU$ is given by the convex hull of the $d$ bilinear functions $u_i$, so we can write each element of $\cP'$ uniquely as a convex combination of this $d$ functions. For a $p'_t \in \cP'$, we will write $p'_{i}$ to be the coefficient of $u_i$ in this convex combination (in this way, we identify $\cP'$ with the simplex $\Delta_{d}$).

Now, for any $p' \in \cP'$ and $\ell' \in \cL'$ (of the above restricted form), we can write

\begin{equation}\label{eq:ip1app}
\langle p', \ell'\rangle = -\sum_{i=d_1+1}^{d_1 + d_2} p'_{i}\left(\frac{1}{d_2}\sum_{j=d_1+1}^{d_1 + d_2}(\ell_{j} - \ell_{i})\right) = \left\langle \pi_2(p'), \pi_2(\ell)\right\rangle - \frac{1}{d_2} \left(\sum_{i=1}^{d_2}\pi_2(p')_{i}\right)\left(\sum_{i=1}^{d_2}\pi_2(\ell)_{i}\right),\end{equation}

\noindent
where $\pi_2 : \Rset^d \rightarrow \Rset^{d_2}$ is the projection map onto the last $d_2$ coordinates. We can simplify this even further by introducing the map $\overline{\pi}: \Delta_d \rightarrow \Delta_{d_2}$ defined via

\begin{equation}
\overline{\pi}(p')_{i} = p'_{d_1 + i} + \frac{1}{d_2}\left(1 - \sum_{j=1}^{d_2}p'_{d_1+j}\right).
\end{equation}

\noindent
This allows us to rewrite \eqref{eq:ip1app} as

\begin{equation}
\langle p', \ell' \rangle = \langle \overline{\pi}(p'), \pi_2(\ell)\rangle - \sum_{j=d_1+1}^{d_2}\ell_{j}.
\end{equation}

Substituting this in turn into \eqref{eq:reg1app}, we have that

\begin{equation}\label{eq:reg2app}
    \Reg(\bp', \bell') = \max_{x^{*} \in \cP'} \left(\sum_{t=1}^{T} \langle \overline{\pi}(p'_t), \pi_2(\ell_t) \rangle - \sum_{t=1}^{T} \left\langle\overline{\pi}(x^*), \pi_2(\ell_t) \right\rangle \right).
\end{equation}

Now, note that the adversary can choose $\ell_t$ so that $\pi_2(\ell_t)$ takes on any value in $\cL_2 = [0, 1]^{d_2}$. Similarly, $\overline{\pi}(p')$ can take on any value in $\cP_2 = \Delta_{d_2}$ as $p'$ ranges over $\cP'$. Therefore, this problem is at least as hard as the online linear optimization problem with action set $\cP_2 = \Delta_{d_2}$ and loss set $\cL_{2} = [0, 1]^{d_2}$. But this is exactly the online learning with experts problem (with $d_2$ experts), which has a regret lower bound of $\Omega(\sqrt{T\log d_2})$. It follows that $\Rate(\cI') \geq \Omega(\sqrt{\log d_2})$, and therefore that $\Rate(\cI') / \Rate(\cI) >  \Omega(\sqrt{(\log d_2)/(\log d_1)} / \eps)$. This quantity can be made arbitrarily large (for a fixed $d_1$ and $d_2$) as we decrease $\eps$. 
\end{proof}

\end{document}